\documentclass{article}

\PassOptionsToPackage{numbers, sort&compress}{natbib}
\PassOptionsToPackage{table}{xcolor}

\usepackage[dblblindworkshop, final]{neurips_2026}
\workshoptitle{AI for Stochastic Dynamics}

\usepackage[utf8]{inputenc}
\usepackage[T1]{fontenc}
\usepackage{hyperref}
\usepackage{url}
\usepackage{booktabs}
\usepackage{multirow}
\usepackage{amsfonts}
\usepackage{amsmath}
\usepackage{amssymb}
\usepackage{mathtools}
\usepackage{graphicx}
\usepackage{nicefrac}
\usepackage{microtype}
\usepackage{placeins}
\usepackage{xcolor}
\definecolor{physhl}{RGB}{254,243,226}
\definecolor{floorhl}{gray}{0.93}

\newcommand{\dtmax}{\Delta t_{\max}}
\newcommand{\Tlyap}{T_{\lambda}}

\usepackage{amsthm}
\theoremstyle{plain}
\makeatletter
\renewenvironment{proof}[1][\proofname]{\par\pushQED{\qed}\normalfont\topsep6\p@\@plus6\p@\relax\trivlist\item[\hskip\labelsep\bfseries#1\@addpunct{.}]\ignorespaces}{\popQED\endtrivlist\@endpefalse}
\makeatother
\usepackage[nameinlink,capitalize,noabbrev]{cleveref}
\newtheorem{proposition}{Proposition}[section]

\newtheorem{lemma}[proposition]{Lemma}
\newif\ifshowcomments
\showcommentsfalse
\ifshowcomments
\newcommand{\jiachen}[1]{\textcolor{orange}{[Jiachen: #1]}}
\newcommand{\alex}[1]{\textcolor{purple}{[Alex: #1]}}
\newcommand{\claude}[1]{\textcolor{teal}{[Claude: #1]}}

\else
\newcommand{\jiachen}[1]{}
\newcommand{\alex}[1]{}
\newcommand{\claude}[1]{}

\fi
\newcommand{\rell}{\mathcal{E}_{L_2}}

\newcommand{\etrans}{\mathcal{E}_{\mathrm{trans}}}
\newcommand{\edtwo}{\mathcal{E}_{D_2}}

\newcommand{\tabfoot}[1]{%
  \sbox0{\scriptsize\color{black!55}\input{#1}}%
  \ifdim\wd0>0pt \par\vspace{1.5pt}\hbox to \linewidth{\usebox0\hfil}\fi}

\title{Physics is the Best Teacher: Consistency Learning for Time-Invariant Operators of Chaotic Dynamics}

\author{%
  Lufang Chiang$^{1,2}$\thanks{Equal contribution. Correspondence to Anima Anandkumar \texttt{<anima@caltech.edu>}.}\hspace{0.4em}\thanks{Work done during the SURF program at Caltech.} \quad
  Jiachen Yao$^{2}$\footnotemark[1] \quad
  Thomas Y.L. Lin$^{3}$ \quad
  Anima Anandkumar$^{2}$ \\
  $^{1}$National Taiwan University \hspace{0.6em}
  $^{2}$California Institute of Technology \hspace{0.6em}
  $^{3}$University of Washington
}

\usepackage{algorithm}
\usepackage{listings}
\definecolor{codeorange}{RGB}{213,94,0}  %
\definecolor{codefunc}{RGB}{0,90,181}    %
\lstdefinelanguage{PythonFuncColor}{
  language=Python,
  keywordstyle=\color{codefunc}\bfseries,
  commentstyle=\color{black!55}\itshape,
  morecomment=[l][\color{codeorange}\bfseries]{\#\#},
  stringstyle=\color{black},
  showstringspaces=false,
  basicstyle=\ttfamily\small,
  literate=
    {standardize}{{\color{codefunc}standardize}}{1}
    {kmeans}{{\color{codefunc}kmeans}}{1}
    {nearest}{{\color{codefunc}nearest}}{1}
    {cluster_radii}{{\color{codefunc}cluster\_radii}}{1}
    {where}{{\color{codefunc}where}}{1}
    {bincount}{{\color{codefunc}bincount}}{1}
    {mean}{{\color{codefunc}mean}}{1}
    {abs}{{\color{codefunc}abs}}{1}
}

\makeatletter
\newskip\origtextfloatsep\origtextfloatsep=\textfloatsep
\newskip\origfloatsep\origfloatsep=\floatsep
\newskip\origintextsep\origintextsep=\intextsep
\newskip\origabovecaptionskip\origabovecaptionskip=\abovecaptionskip
\newskip\origbelowcaptionskip\origbelowcaptionskip=\belowcaptionskip
\newskip\origparskip\origparskip=\parskip
\newskip\origtopsep\origtopsep=\topsep
\newskip\origitemsep\origitemsep=\itemsep
\newskip\origparsep\origparsep=\parsep
\let\orignormalsize\normalsize
\let\origsection\section
\let\origparagraph\paragraph
\let\orig@listi\@listi
\let\orig@listI\@listI
\newcommand{\restorespacing}{%
  \setlength{\textfloatsep}{\origtextfloatsep}\setlength{\floatsep}{\origfloatsep}\setlength{\intextsep}{\origintextsep}%
  \setlength{\abovecaptionskip}{\origabovecaptionskip}\setlength{\belowcaptionskip}{\origbelowcaptionskip}%
  \setlength{\parskip}{\origparskip}%
  \setlength{\topsep}{\origtopsep}\setlength{\itemsep}{\origitemsep}\setlength{\parsep}{\origparsep}%
  \let\normalsize\orignormalsize\let\section\origsection\let\paragraph\origparagraph
  \let\@listi\orig@listi\let\@listI\orig@listI\normalsize}
\makeatother
\makeatletter\g@addto@macro\normalsize{\setlength{\abovedisplayskip}{4pt plus 1pt minus 1pt}\setlength{\belowdisplayskip}{4pt plus 1pt minus 1pt}\setlength{\abovedisplayshortskip}{2pt}\setlength{\belowdisplayshortskip}{2pt}}\makeatother
\makeatletter
\renewcommand{\section}{\@startsection{section}{1}{\z@}{-1.4ex \@plus -0.3ex \@minus -0.2ex}{0.9ex \@plus 0.2ex \@minus 0.2ex}{\large\bf\raggedright}}
\renewcommand{\paragraph}{\@startsection{paragraph}{4}{\z@}{0.9ex \@plus 0.3ex \@minus 0.2ex}{-1em}{\normalsize\bf}}
\makeatother
\makeatletter\def\@listi{\leftmargin\leftmargini \topsep 2pt \parsep 0pt \itemsep 1pt}\let\@listI\@listi\@listi\makeatother
\begin{document}

\maketitle

\begin{abstract}
Accelerating the prediction of long-term behavior in chaotic systems is crucial in scientific computing.
However, existing methods rely on numerical solvers or autoregressive models that advance one small step at a time, which makes long horizons expensive.
We instead view this problem as learning the system's time-invariant evolution operator, which jumps the state across a large time span in a single evaluation.
To this end, we derive the consistency equations a time-invariant operator must satisfy, with differential and compositional objectives in physical time.
These equations also connect the learned operator to the physics-prescribed instant dynamics, enabling physics embedding in consistency learning.
Across five chaotic systems, we find that physics-distilled consistency makes both short-term trajectories and long-term statistics more accurate.
The learned operator survives temporal extrapolation and requires one-tenth as many evaluations as autoregressive rollout, offering an efficient route to long-term simulation of chaotic dynamics.

\end{abstract}

\section{Introduction}
\label{sec:intro}
\suppressfloats[t]%
\begin{figure}[t]
\centering
\includegraphics[width=\linewidth]{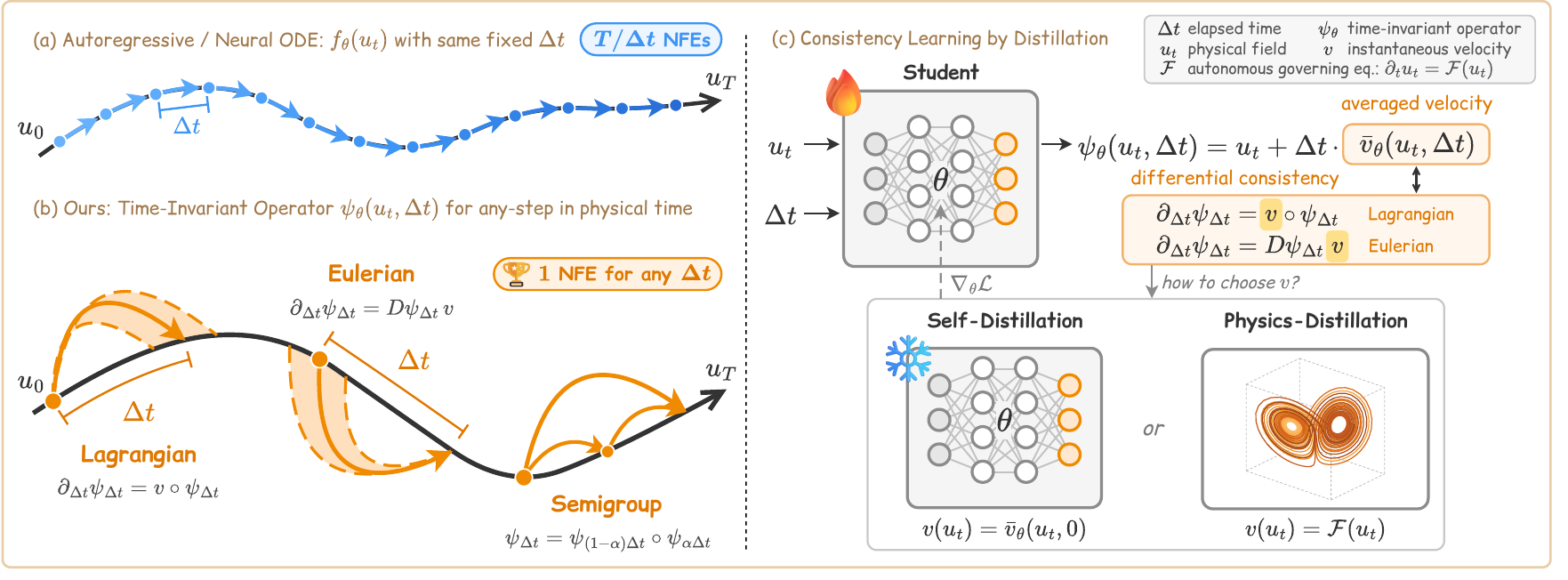}
\vspace{-1em}
\caption{\textbf{Overview.} (a) Autoregressive models and neural ODEs advance in small steps and need many NFEs. (b) We propose consistency conditions for a time-invariant operator that jumps any $\Delta t$ in one step. (c) The model itself or the governing equations supply the velocity in these conditions.}
\label{fig:main}
\end{figure}
Predicting the long-term behavior of chaotic systems is a core task of scientific computing, from weather forecasting \citep{pathak2022fourcastnet, lam2023graphcast, bi2023pangu} to the simulation of fluids \citep{kochkov2021machine, pfaff2021learning}.
Chaotic systems amplify perturbations exponentially \citep{sprott2003chaos, gilpin2021chaos}, so a numerical integrator has to take steps small enough that the discretization error does not grow much along the trajectory, and the cost of simulations over a long horizon is therefore prohibitive.
An operator trained at a fixed step and rolled out autoregressively \citep{li2020fourier, brandstetter2022message, ye2025recurrentneuraloperatorsstable} spends as many function evaluations (NFE) as the solver takes steps, and the error compounds across them.
A neural ODE or a Continuous Flow Operator (CFO) \citep{chen2018neural, hou2025cfolearningcontinuoustimepde} replaces the fixed step with an integrator over a learned instantaneous velocity, but still trades solver precision against many NFEs.

To accelerate the prediction, one can view the problem as learning a time-invariant evolution operator of an autonomous system.
For autonomous dynamics $\partial_t u_t=v(u_t)$, this operator takes a state and an elapsed time to the future state, i.e., $\psi(u_t,\Delta t)=u_{t+\Delta t}$.
Because $\psi$ depends on elapsed rather than absolute time, the same model applies from every starting time and can jump directly to any requested horizon in a single evaluation.
Recent work learns related operators for molecular and physical dynamics \citep{bai2026physics, ripken2026hamiltonian, luo2026recovering}, while one-step generative models use a similar idea to compress many sampling steps into one \citep{song2023consistencymodels, boffi2025flowmapmatchingstochastic, geng2025meanflowsonestepgenerative, wang2025transition}.
Across these settings, the shared computational goal is to replace repeated integration with direct evaluation of a learned operator.
Learning that operator is especially demanding for chaotic dynamics \citep{hess2023generalized, wang2024closure, zhang2025zeroshot}: it must represent strong sensitivity to the input state while remaining consistent with one autonomous law over a continuum of horizons \citep{koppe2019identifying, durstewitz2026position}.
What remains unclear is how to formulate these consistency conditions systematically in physical time, and how much imposing them buys over fitting trajectory data alone.

Fitting trajectory data alone does not ensure that the learned operator follows the same dynamical law across all horizons.
A finite set of trajectories directly constrains the operator only at observed states and sampled time intervals, leaving its behavior at intermediate horizons underdetermined.
An autonomous flow must also satisfy two structural conditions: splitting a time interval into two steps must not change the predicted state \citep{luo2026recovering, shikhman2026semigroup}, and the derivative with respect to the horizon must equal the system's velocity field evaluated at the evolved state \citep{boffi2025consistency}.
Regression alone guarantees neither condition, so even a model that fits the observed transitions need not represent the flow of an autonomous differential equation \citep{shikhman2026semigroup}.
Enforcing these conditions links predictions across horizons and provides supervision beyond the sampled transitions.

The consistency conditions give physics a teaching role in learning time-invariant evolution operators.
When the governing equations are known, their velocity field supplies the instantaneous velocity needed to construct targets for the differential consistency objectives.
Physics thus enters the consistency targets, shifting the methodology from hard constraints and PINN-style equation residuals \citep{raissi2019physics} to the supervision source.
Alternatively, the model can supply its own velocity estimate from its learned ``average'' velocity at zero horizon.
Both choices use the same consistency objectives for distillation \citep{boffi2025consistency},
which lets us ask whether physics is a better teacher than the model itself.

In summary, our contributions are the following.
\begin{itemize}\itemsep 2pt
    \item We explore the design space of consistency learning for time-invariant evolution operators by deriving two differential forms and a compositional form in physical time, each yielding a training objective on the averaged velocity field to promote stability.
    \item We instantiate the same differential consistency objectives with either the model or the governing physics as the teacher. Through controlled comparisons, we show that physics-distillation consistently outperforms self-distillation.
    \item We compare all the design choices on five representative chaotic systems, inside and beyond the training horizon, and find that the Eulerian form with a physics teacher stands out.
\end{itemize}

\section{Consistency Learning for Time-Invariant Operators}
\label{sec:method}
We consider autonomous dynamics $\partial_t u_t = v(u_t)$ on a state space $\mathcal{U}$, where $v$ is the instantaneous velocity field.
The time-invariant evolution operator $\psi : \mathcal{U} \times [0,\dtmax] \to \mathcal{U}$ sends a state and an elapsed time to the state reached after advancing by that much.
Here, we follow \citep{geng2025meanflowsonestepgenerative, boffi2025flowmapmatchingstochastic, ripken2026hamiltonian} and parameterize it through the averaged velocity $\bar{v}$, the mean of $v$ over the elapsed time,
\begin{equation}
\psi_{\theta}(u,\Delta t) \triangleq u + \Delta t\,\bar{v}_{\theta}(u,\Delta t).
\label{eq:parameterization}
\end{equation}
In practice, the goal is to learn the evolution operator $\psi_{\theta}$ from collected trajectories, ensuring we can supervise $\bar{v}_{\theta}$ directly through the secant between two frames,
\begin{equation}
\mathcal{L}_{\mathrm{data}}(\theta)
= \mathbb{E}_{u_t,\,\Delta t}\Big[\big\|\bar{v}_{\theta}(u_t,\Delta t)-\tfrac{u_{t+\Delta t}-u_t}{\Delta t}\big\|^2\Big].
\label{eq:data}
\end{equation}

\paragraph{Consistency of a time-invariant operator.}
For a time-invariant operator $\psi$, the prediction satisfies time-shift invariance and does not depend on absolute time: from the state the trajectory has reached after $s$, advancing by the remaining $\Delta t - s$ lands where advancing by $\Delta t$ from the start does.
Differentiating that statement in $s$ ties the horizon derivative of $\psi$ to the velocity field, in two equivalent ways depending on whether the velocity is read at the start or at the end of the jump, and writing that intermediate state as $\psi(u_t,s)$ instead gives the composition rule.
In terms of the averaged velocity, and following \citep{boffi2025flowmapmatchingstochastic}, the operator satisfies the three consistency conditions below:
\begin{align}
\textit{Eulerian:} && \bar{v}(u_t,\Delta t) &= v(u_t) + \Delta t\big(D\bar{v}(u_t,\Delta t)[v(u_t)] - \partial_{\Delta t}\bar{v}(u_t,\Delta t)\big), \label{eq:eul}\\
\textit{Lagrangian:} && \bar{v}(u_t,\Delta t) &= v\big(u_t + \Delta t\,\bar{v}(u_t,\Delta t)\big) - \Delta t\,\partial_{\Delta t}\bar{v}(u_t,\Delta t), \label{eq:lag}\\
\textit{Semigroup:} && \psi(u_t,\Delta t) &= \psi\big(\psi(u_t,\alpha\Delta t),(1-\alpha)\Delta t\big), \quad \alpha \in [0,1]. \label{eq:semi}
\end{align}
\cref{sec:theory-consistency} states the assumptions and proves all three.
What remains is how to impose them on a learned operator, since fitting \cref{eq:data} constrains $\bar{v}_{\theta}$ only through the endpoint of each jump, without providing any other supervision on the structure in between.

\paragraph{Imposing the conditions.}
For differential consistency, \cref{eq:eul,eq:lag} both call for the instantaneous velocity $v$, and the operator can supply $v$ itself, since the tangent condition $\bar{v}(u,0) = v(u)$ holds by construction (\cref{lem:boundary-condition}), which makes the objective a \emph{self-distillation}.
When the governing equations $\partial_t u_t = v(u_t) = \mathcal{F}(u_t)$ are known, $v$ can be read from $\mathcal{F}$ instead, which makes it a \emph{physics-distillation}.
This gives \emph{Eulerian-self} and \emph{Lagrangian-self} against \emph{Eulerian-phys} and \emph{Lagrangian-phys} for elucidating the velocity teacher, and \emph{Semigroup} as the point in the design space with no velocity guidance at all.
To see which condition is the better thing to learn from, we square the residual of each and add it to the data term,
\begin{equation}
\mathcal{L}(\theta) = \mathcal{L}_{\mathrm{data}}(\theta) + w\,\mathcal{L}_{\mathrm{cons}}(\theta),
\label{eq:total}
\end{equation}
where $\mathcal{L}_{\mathrm{cons}}$ is the residual of the consistency equation and $w$ controls its weight.
Details on the horizon weighting and the stop-gradient are given in \cref{sec:theory-consistency}; \cref{app:ablations} sweeps how sharply that weighting grows with the horizon; \cref{app:connections} relates consistency conditions to previous works.

\section{Experiments}
\label{sec:results}
\begin{table}[t]
\caption{\textbf{Average prediction results over the five systems.}
Five metrics averaged over three seeds and over the five systems by geometric mean, with $D_{\mathrm{H}}$ and $\edtwo$ in $10^{-2}$ and $\etrans$ in $10^{-3}$.
Out-of-distribution means the model is trained on only half of the horizon and rolled out fully.
Shading marks the physics-distilled rows, and bold the best method in each column of a block.
\cref{app:metrics,app:details} define the metrics and the training setup, and \cref{app:tables} gives the table system by system.}
\label{tab:main-avg}
\centering
\small
\setlength{\tabcolsep}{12pt}
\resizebox{\linewidth}{!}{\begin{tabular}{lrcccccccccc}
\toprule
 & & \multicolumn{5}{c}{\emph{in distribution}} & \multicolumn{5}{c}{\emph{out of distribution}} \\
\cmidrule(lr){3-7} \cmidrule(lr){8-12}
Method & NFE & $\rell$ & $D_{\mathrm{stsp}}$ & $D_{\mathrm{H}}$ & $\edtwo$ & $\etrans$ & $\rell$ & $D_{\mathrm{stsp}}$ & $D_{\mathrm{H}}$ & $\edtwo$ & $\etrans$ \\
\midrule
\rowcolor{black!8}\multicolumn{12}{l}{\emph{Evolution operator}} \\
Data & 100 & 0.46 & 0.20 & 3.1 & 2.5 & 1.7 & 0.44 & 0.18 & 2.9 & 1.6 & 1.1 \\
Semigroup & 100 & 0.46 & 0.21 & 3.4 & 3.3 & 1.9 & 0.44 & 0.18 & 2.9 & 1.3 & 1.3 \\
Eulerian-self & 100 & 0.45 & 0.22 & 3.5 & 3.8 & 1.6 & 0.43 & \textbf{0.17} & 2.6 & 1.6 & 1.4 \\
Lagrangian-self & 100 & 0.45 & 0.21 & 3.4 & 3.2 & 1.5 & 0.48 & 0.18 & 2.9 & 2.0 & \textbf{0.94} \\
\rowcolor{physhl}Eulerian-phys & 100 & \textbf{0.28} & \textbf{0.17} & \textbf{2.0} & \textbf{0.94} & \textbf{0.93} & \textbf{0.27} & \textbf{0.17} & 2.1 & \textbf{0.87} & 1.1 \\
\rowcolor{physhl}Lagrangian-phys & 100 & 0.31 & 0.22 & 2.7 & 1.9 & 1.2 & 0.30 & \textbf{0.17} & \textbf{2.0} & 1.4 & 1.1 \\
\midrule
\rowcolor{black!8}\multicolumn{12}{l}{\emph{Baselines}} \\
AR $v$-pred & 1000 & \textbf{0.27} & \textbf{0.17} & \textbf{1.8} & \textbf{1.0} & \textbf{0.48} & \textbf{0.26} & \textbf{0.17} & \textbf{1.6} & \textbf{1.2} & \textbf{0.50} \\
AR $x$-pred (cond.\ $t$) & 1000 & 0.64 & 0.35 & 4.7 & 5.3 & 2.1 & 1.2 & 3.8 & 28 & 35 & 19 \\
AR $x$-pred (uncond.\ $t$) & 1000 & 0.59 & 0.24 & 4.3 & 3.5 & 1.6 & 0.57 & 0.23 & 3.8 & 2.8 & 1.3 \\
CFO & 4000 & 0.30 & \textbf{0.17} & 1.9 & \textbf{1.0} & 0.82 & 1.2 & 1.7 & 12 & 10 & 4.9 \\
\bottomrule
\end{tabular}
}%
\vspace{-0.5em}
\end{table}

\paragraph{Setup.}
We use five chaotic systems from the dysts catalogue \citep{gilpin2021chaos}, Lorenz, Chen, R\"ossler, Thomas and Halvorsen, measuring every metric in their Lyapunov times $\Tlyap$.
Each dataset holds $10{,}000$ trajectories of $5\Tlyap$, detailed in \cref{app:datasets}, and every method shares one MLP with residual connections, with or without conditioning on time, with hyperparameters tuned by Vizier \citep{golovin2017vizier} on the validation split and every cell run over three seeds, detailed in \cref{app:details}.
The baselines are an autoregressive model (AR) that predicts either the next state ($x$-pred) or the velocity ($v$-pred), and CFO \citep{hou2025cfolearningcontinuoustimepde}, a continuous-time model that predicts instantaneous velocity, trained on a polynomial interpolant of the data points.
For a compact overview, \cref{app:methods} lists each method's defining equation, NFE, and average prediction error.
Rolling out five Lyapunov times costs the evolution operator $100$ evaluations, the AR baselines $1000$, and CFO $4000$.
We first train on the full trajectory and evaluate at the endpoint (in distribution), and then train on the first half and evaluate at the endpoint of the full trajectory (out of distribution).
Besides the relative $L_2$ error for short-term accuracy, we report long-term metrics on the distribution the model settles into, detailed in \cref{app:metrics}.

\paragraph{Choice of teacher.}
\cref{tab:main-avg,tab:5lyap,tab:1lyap} compare the five consistency objectives against the data-only control.
The physics-distillation objectives improve on the control by a factor of $1.6$ at $5\Tlyap$ and a factor of $3$ at $\Tlyap$, and the same holds for the long-term statistics, while self-distillation stays level with it.
The difference lies in how each teacher constrains the operator.
The consistency conditions in \cref{sec:method} require the operator to be the flow of some velocity field, but self-distillation velocity does not strictly enforce that the velocity matches the physical one.
Substituting the governing equations selects a single member of that family, and constrains the operator at every horizon rather than only at the frames the data samples.
\cref{fig:dose} is consistent with this reading, as the physics curves fall below the control while the weight grows and the self-distilled ones rise above it.
The operator also needs $10\times$ fewer evaluations than the best baseline, and the two baselines that use absolute time break down out of distribution, while the same network without the time input does not.
\cref{fig:qual-main} shows one trajectory past the training horizon, where the operator trained with physics-distillation stays on the truth and CFO leaves it first.

\paragraph{Choice of form.}
Under a physics teacher, Eulerian is ahead of Lagrangian, though the two stay close and most cells of \cref{tab:5lyap} do not separate them, and what differs between the two is where $\mathcal{F}$ is evaluated.
Eulerian evaluates $\mathcal{F}$ at the data state, so its target is fixed by the data and does not move while the operator trains.
Lagrangian evaluates $\mathcal{F}$ at the operator's own predicted endpoint, so its target is bootstrapped on the current prediction and follows the model's own error, which may be the harder signal to learn from.
The choice of form matters far less than the choice of teacher.

\begin{table}[t]
\caption{\textbf{Five-Lyapunov-time accuracy.}
$\rell(5\Tlyap)$ on each of the five systems, mean $\pm$ std over three seeds, with each method at its native number of evaluations, stated in the NFE column.
In-distribution means every model is trained on the full horizon, and out-of-distribution means it is trained on the first half and rolled out to the full horizon.
Shading marks the physics-distilled objectives and bold the best method within each block.
\cref{app:tables} gives the same table at one Lyapunov time.}
\label{tab:5lyap}
\centering
\small
\setlength{\tabcolsep}{3pt}
\resizebox{\linewidth}{!}{\begin{tabular}{lrcccccccccc}
\toprule
 & & \multicolumn{2}{c}{Lorenz} & \multicolumn{2}{c}{Chen} & \multicolumn{2}{c}{R\"ossler} & \multicolumn{2}{c}{Thomas} & \multicolumn{2}{c}{Halvorsen} \\
\cmidrule(lr){3-4} \cmidrule(lr){5-6} \cmidrule(lr){7-8} \cmidrule(lr){9-10} \cmidrule(lr){11-12}
Method & NFE & in-dist & OOD & in-dist & OOD & in-dist & OOD & in-dist & OOD & in-dist & OOD \\
\midrule
\rowcolor{black!8}\multicolumn{12}{l}{\emph{Evolution operator}} \\
Data & 100 & 0.32{\scriptsize\color{black!55}$\pm$0.05} & 0.25{\scriptsize\color{black!55}$\pm$0.04} & 0.26{\scriptsize\color{black!55}$\pm$0.03} & 0.27{\scriptsize\color{black!55}$\pm$0.03} & 0.53{\scriptsize\color{black!55}$\pm$0.06} & 0.53{\scriptsize\color{black!55}$\pm$0.02} & 0.90{\scriptsize\color{black!55}$\pm$0.04} & 0.92{\scriptsize\color{black!55}$\pm$0.04} & 0.56{\scriptsize\color{black!55}$\pm$0.09} & 0.52{\scriptsize\color{black!55}$\pm$0.01} \\
Semigroup & 100 & 0.30{\scriptsize\color{black!55}$\pm$0.06} & 0.24{\scriptsize\color{black!55}$\pm$0.03} & 0.27{\scriptsize\color{black!55}$\pm$0.02} & 0.29{\scriptsize\color{black!55}$\pm$0.04} & 0.54{\scriptsize\color{black!55}$\pm$0.05} & 0.54{\scriptsize\color{black!55}$\pm$0.01} & 0.88{\scriptsize\color{black!55}$\pm$0.09} & 0.90{\scriptsize\color{black!55}$\pm$0.07} & 0.55{\scriptsize\color{black!55}$\pm$0.09} & 0.52{\scriptsize\color{black!55}$\pm$0.03} \\
Eulerian-self & 100 & 0.27{\scriptsize\color{black!55}$\pm$0.06} & 0.24{\scriptsize\color{black!55}$\pm$0.01} & 0.27{\scriptsize\color{black!55}$\pm$0.03} & 0.30{\scriptsize\color{black!55}$\pm$0.04} & 0.56{\scriptsize\color{black!55}$\pm$0.02} & 0.50{\scriptsize\color{black!55}$\pm$0.04} & 0.89{\scriptsize\color{black!55}$\pm$0.07} & 0.83{\scriptsize\color{black!55}$\pm$0.05} & 0.50{\scriptsize\color{black!55}$\pm$0.02} & 0.49{\scriptsize\color{black!55}$\pm$0.02} \\
Lagrangian-self & 100 & 0.22{\scriptsize\color{black!55}$\pm$0.03} & 0.31{\scriptsize\color{black!55}$\pm$0.14} & 0.31{\scriptsize\color{black!55}$\pm$0.05} & 0.32{\scriptsize\color{black!55}$\pm$0.06} & 0.56{\scriptsize\color{black!55}$\pm$0.06} & 0.54{\scriptsize\color{black!55}$\pm$0.03} & 0.85{\scriptsize\color{black!55}$\pm$0.15} & 0.87{\scriptsize\color{black!55}$\pm$0.04} & 0.55{\scriptsize\color{black!55}$\pm$0.07} & 0.56{\scriptsize\color{black!55}$\pm$0.02} \\
\rowcolor{physhl}Eulerian-phys & 100 & \textbf{0.14{\scriptsize\color{black!55}$\pm$0.06}} & 0.13{\scriptsize\color{black!55}$\pm$0.03} & \textbf{0.15{\scriptsize\color{black!55}$\pm$0.01}} & \textbf{0.13{\scriptsize\color{black!55}$\pm$0.02}} & \textbf{0.45{\scriptsize\color{black!55}$\pm$0.03}} & \textbf{0.44{\scriptsize\color{black!55}$\pm$0.03}} & 0.64{\scriptsize\color{black!55}$\pm$0.07} & \textbf{0.61{\scriptsize\color{black!55}$\pm$0.06}} & \textbf{0.31{\scriptsize\color{black!55}$\pm$0.01}} & \textbf{0.32{\scriptsize\color{black!55}$\pm$0.02}} \\
\rowcolor{physhl}Lagrangian-phys & 100 & \textbf{0.14{\scriptsize\color{black!55}$\pm$0.01}} & \textbf{0.12{\scriptsize\color{black!55}$\pm$0.02}} & \textbf{0.15{\scriptsize\color{black!55}$\pm$0.01}} & 0.17{\scriptsize\color{black!55}$\pm$0.04} & 0.52{\scriptsize\color{black!55}$\pm$0.06} & 0.51{\scriptsize\color{black!55}$\pm$0.04} & \textbf{0.62{\scriptsize\color{black!55}$\pm$0.02}} & 0.62{\scriptsize\color{black!55}$\pm$0.05} & 0.39{\scriptsize\color{black!55}$\pm$0.03} & 0.37{\scriptsize\color{black!55}$\pm$0.05} \\
\midrule
\rowcolor{black!8}\multicolumn{12}{l}{\emph{Baselines}} \\
AR $v$-pred & 1000 & \textbf{0.11{\scriptsize\color{black!55}$\pm$0.03}} & \textbf{0.094{\scriptsize\color{black!55}$\pm$0.023}} & \textbf{0.19{\scriptsize\color{black!55}$\pm$0.05}} & \textbf{0.21{\scriptsize\color{black!55}$\pm$0.04}} & \textbf{0.27{\scriptsize\color{black!55}$\pm$0.06}} & \textbf{0.30{\scriptsize\color{black!55}$\pm$0.07}} & 0.82{\scriptsize\color{black!55}$\pm$0.09} & \textbf{0.78{\scriptsize\color{black!55}$\pm$0.06}} & 0.27{\scriptsize\color{black!55}$\pm$0.02} & \textbf{0.25{\scriptsize\color{black!55}$\pm$0.02}} \\
AR $x$-pred (cond.\ $t$) & 1000 & 0.50{\scriptsize\color{black!55}$\pm$0.07} & 0.75{\scriptsize\color{black!55}$\pm$0.08} & 0.45{\scriptsize\color{black!55}$\pm$0.07} & 1.0{\scriptsize\color{black!55}$\pm$0.5} & 0.61{\scriptsize\color{black!55}$\pm$0.01} & 0.71{\scriptsize\color{black!55}$\pm$0.11} & 1.1{\scriptsize\color{black!55}$\pm$0.1} & 2.1{\scriptsize\color{black!55}$\pm$0.6} & 0.68{\scriptsize\color{black!55}$\pm$0.12} & 1.8{\scriptsize\color{black!55}$\pm$0.4} \\
AR $x$-pred (uncond.\ $t$) & 1000 & 0.42{\scriptsize\color{black!55}$\pm$0.12} & 0.39{\scriptsize\color{black!55}$\pm$0.12} & 0.47{\scriptsize\color{black!55}$\pm$0.03} & 0.45{\scriptsize\color{black!55}$\pm$0.07} & 0.49{\scriptsize\color{black!55}$\pm$0.05} & 0.57{\scriptsize\color{black!55}$\pm$0.02} & 1.1{\scriptsize\color{black!55}$\pm$0.1} & 1.1{\scriptsize\color{black!55}$\pm$0.0} & 0.72{\scriptsize\color{black!55}$\pm$0.12} & 0.57{\scriptsize\color{black!55}$\pm$0.03} \\
CFO & 4000 & 0.14{\scriptsize\color{black!55}$\pm$0.02} & 0.55{\scriptsize\color{black!55}$\pm$0.06} & \textbf{0.19{\scriptsize\color{black!55}$\pm$0.04}} & 0.30{\scriptsize\color{black!55}$\pm$0.03} & 0.47{\scriptsize\color{black!55}$\pm$0.05} & 8.7{\scriptsize\color{black!55}$\pm$5.8} & \textbf{0.80{\scriptsize\color{black!55}$\pm$0.07}} & 2.1{\scriptsize\color{black!55}$\pm$0.5} & \textbf{0.25{\scriptsize\color{black!55}$\pm$0.03}} & 0.88{\scriptsize\color{black!55}$\pm$0.08} \\
\bottomrule
\end{tabular}
}%
\end{table}

\begin{figure}[t]
\vspace{1em}
\centering
\begin{minipage}[t]{0.4083\linewidth}
\centering
\includegraphics[width=\linewidth]{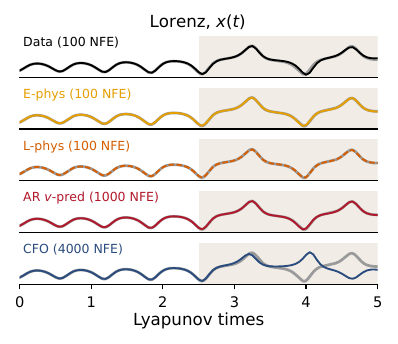}%
\vspace{-1.8em}
\caption{\textbf{Temporal extrapolation.}
One test trajectory's $x$-coordinate on Lorenz, the truth in gray and one method per row.
Every model is trained on the first half only, so no model saw the shaded half.}
\label{fig:qual-main}
\end{minipage}\hfill
\begin{minipage}[t]{0.5637\linewidth}
\centering
\raisebox{4.92pt}[\height][\depth]{\includegraphics[width=\linewidth]{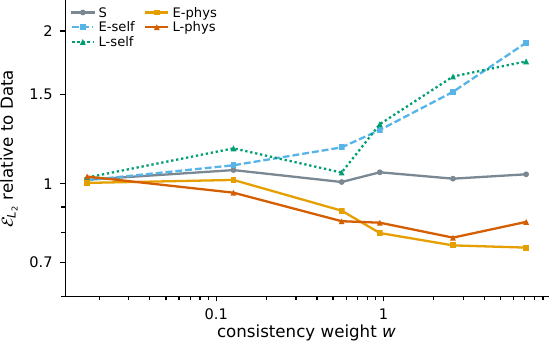}}%
\vspace{-1.8em}
\caption{\textbf{Consistency strength.}
Validation relative L2 error against the consistency weight, as a fraction of the data-only model's error, averaged over the five systems.
E-phys and L-phys are fed the governing equations; the others are fed the model's own predictions.}
\label{fig:dose}
\end{minipage}
\vspace{-0.5em}
\end{figure}

\FloatBarrier
\section{Conclusion}
\label{sec:discussion}
In this work, we learn the time-invariant evolution operator of an autonomous system from the consistency conditions it must satisfy, with either the model itself or the governing equations as the teacher.
Across five chaotic systems, self-distillation carries little supervision and lands on the data-only control, while physics-distillation improves on it in and beyond the training horizon, where the baselines that read absolute time break down.
While we lay out this design space, we do not test every choice in it, and leave a direct test of what separates the two differential forms to future work.

\newpage
\section*{Declaration of LLM Usage}
\label{sec:llm}
Large language models (LLMs) were used to assist with writing and formatting, to examine proofs, and to support code development.
All LLM-assisted text, proof assessments, and code were subsequently reviewed by the authors.
LLMs did not contribute to the development of the core methods in this research.

\section*{Acknowledgments and Disclosure of Funding}
\label{sec:acknowledgments}
Anima Anandkumar is supported in part by Bren endowed chair, ONR (MURI grant N00014-23-1-2654), and the AI2050 senior fellow program at Schmidt Sciences. Lufang Chiang is supported in part by the Babar SURF Fellowship. 
Jiachen Yao is supported in part by the Naren and Vinita Gupta Fellowship. 

\bibliographystyle{unsrtnat}
\bibliography{reference}

@inproceedings{li2020fourier,
  author    = {Zongyi Li and Nikola Borislavov Kovachki and Kamyar Azizzadenesheli and
               Burigede Liu and Kaushik Bhattacharya and Andrew M. Stuart and Anima Anandkumar},
  title     = {Fourier Neural Operator for Parametric Partial Differential Equations},
  booktitle = {International Conference on Learning Representations},
  year      = {2021},
  url       = {https://openreview.net/forum?id=c8P9NQVtmnO}
}

@article{kovachki2023neural,
  title   = {Neural operator: Learning maps between function spaces with applications to {PDEs}},
  author  = {Kovachki, Nikola and Li, Zongyi and Liu, Burigede and Azizzadenesheli, Kamyar and Bhattacharya, Kaushik and Stuart, Andrew and Anandkumar, Anima},
  journal = {Journal of Machine Learning Research},
  volume  = {24},
  number  = {89},
  pages   = {1--97},
  year    = {2023},
  url = {https://arxiv.org/abs/2108.08481}
}

@inproceedings{hou2025cfolearningcontinuoustimepde,
  title     = {{CFO}: Learning Continuous-Time {PDE} Dynamics via Flow-Matched Neural Operators},
  author    = {Hou, Xianglong and Huang, Xinquan and Perdikaris, Paris},
  booktitle = {International Conference on Learning Representations},
  year      = {2026},
  url       = {https://openreview.net/forum?id=IQhaeSzyup}
}

@misc{ye2025recurrentneuraloperatorsstable,
      title={Recurrent Neural Operators: Stable Long-Term {PDE} Prediction}, 
      author={Zaijun Ye and Chen-Song Zhang and Wansheng Wang},
      year={2025},
      eprint={2505.20721},
      archivePrefix={arXiv},
      primaryClass={cs.LG},
      url={https://arxiv.org/abs/2505.20721}, 

}

@inproceedings{geng2025meanflowsonestepgenerative,
  author    = {Geng, Zhengyang and Deng, Mingyang and Bai, Xingjian and Kolter, Zico and He, Kaiming},
  title     = {Mean Flows for One-step Generative Modeling},
  booktitle = {Advances in Neural Information Processing Systems},
  volume    = {38},
  pages     = {84032--84054},
  year      = {2025},
  doi       = {10.52202/085713-2534},
  url       = {https://proceedings.neurips.cc/paper_files/paper/2025/hash/6d13e085b79d454da5910e4ca82a3d9d-Abstract-Conference.html}
}

@inproceedings{song2023consistencymodels,
  title     = {Consistency Models},
  author    = {Song, Yang and Dhariwal, Prafulla and Chen, Mark and Sutskever, Ilya},
  booktitle = {Proceedings of the 40th International Conference on Machine Learning},
  series    = {Proceedings of Machine Learning Research},
  volume    = {202},
  pages     = {32211--32252},
  year      = {2023},
  publisher = {PMLR},
  url       = {https://proceedings.mlr.press/v202/song23a.html}
}

@article{boffi2025flowmapmatchingstochastic,
  author  = {Nicholas M. Boffi and Michael S. Albergo and Eric Vanden{-}Eijnden},
  title   = {Flow map matching with stochastic interpolants: {A} mathematical framework for consistency models},
  journal = {Transactions on Machine Learning Research},
  year    = {2025},
  url     = {https://openreview.net/forum?id=cqDH0e6ak2}
}

@article{lloyd1982least,
  title   = {Least Squares Quantization in {PCM}},
  author  = {Lloyd, Stuart P.},
  journal = {IEEE Transactions on Information Theory},
  volume  = {28},
  number  = {2},
  pages   = {129--137},
  year    = {1982},
  url = {https://doi.org/10.1109/TIT.1982.1056489}
}

@inproceedings{wang2024closure,
  title     = {Beyond Closure Models: Learning Chaotic-Systems via Physics-Informed Neural Operators},
  author    = {Wang, Chuwei and Berner, Julius and Li, Zongyi and Zhou, Di and Wang, Jiayun and Bae, Jane and Anandkumar, Anima},
  booktitle = {D3S3: Data-driven and Differentiable Simulations, Surrogates, and Solvers Workshop @ NeurIPS},
  year      = {2024},
  url = {https://arxiv.org/abs/2408.05177v4}
}

@book{hairer1993solving,
  title     = {Solving Ordinary Differential Equations I: Nonstiff Problems},
  author    = {Hairer, Ernst and N{\o}rsett, Syvert P. and Wanner, Gerhard},
  series    = {Springer Series in Computational Mathematics},
  volume    = {8},
  edition   = {2nd},
  publisher = {Springer},
  address   = {Berlin, Heidelberg},
  year      = {1993},
  doi       = {10.1007/978-3-540-78862-1},
  url = {https://doi.org/10.1007/978-3-540-78862-1}
}

@inproceedings{gilpin2021chaos,
  title={Chaos as an interpretable benchmark for forecasting and data-driven modelling},
  author={Gilpin, William},
  booktitle={Advances in Neural Information Processing Systems, Datasets and Benchmarks Track},
  year={2021},
  url = {https://arxiv.org/abs/2110.05266}
}

@inproceedings{zhang2025zeroshot,
  title={Zero-shot forecasting of chaotic systems},
  author={Zhang, Yuanzhao and Gilpin, William},
  booktitle={International Conference on Learning Representations},
  year={2025},
  url = {https://arxiv.org/abs/2409.15771}
}

@article{prinz2011markov,
  title={Markov models of molecular kinetics: Generation and validation},
  author={Prinz, Jan-Hendrik and Wu, Hao and Sarich, Marco and Keller, Bettina and Senne, Martin and Held, Martin and Chodera, John D. and Sch{\"u}tte, Christof and No{\'e}, Frank},
  journal={The Journal of Chemical Physics},
  volume={134},
  pages={174105},
  year={2011},
  url = {https://doi.org/10.1063/1.3565032}
}

@book{ulam1960collection,
  title={A Collection of Mathematical Problems},
  author={Ulam, Stanislaw M.},
  publisher={Interscience Publishers},
  year={1960}

}

@article{lorenz1963deterministic,
      title={Deterministic Nonperiodic Flow},
      author={Edward N. Lorenz},
      journal={Journal of the Atmospheric Sciences},
      volume={20},
      number={2},
      pages={130--141},
      year={1963},
  url = {https://doi.org/10.1175/1520-0469(1963)020<0130:DNF>2.0.CO;2}
}

@misc{bai2026physics,
      title={Physics-Informed Neural {PDE} Solvers via Spatio-Temporal {MeanFlow}},
      author={Hanru Bai and Yuncheng Zhou and Difan Zou},
      year={2026},
      eprint={2605.08915},
      archivePrefix={arXiv},
      primaryClass={cs.LG},
      url={https://arxiv.org/abs/2605.08915},

}

@inproceedings{ripken2026hamiltonian,
  title         = {Learning {H}amiltonian Flow Maps: Mean Flow Consistency for Large-Timestep Molecular Dynamics},
  author        = {Ripken, Winfried and Plainer, Michael and Lied, Gregor and Frank, Thorben and Unke, Oliver T. and Chmiela, Stefan and No{\'e}, Frank and M{\"u}ller, Klaus-Robert},
  booktitle     = {Proceedings of the 43rd International Conference on Machine Learning},
  publisher     = {PMLR},
  year          = {2026},
  eprint        = {2601.22123},
  archivePrefix = {arXiv},
  primaryClass  = {cs.LG}
}

@misc{luo2026recovering,
      title={Recovering Physical Dynamics from Discrete Observations via Intrinsic Differential Consistency},
      author={Yuxiang Luo and Andrew Perrault},
      year={2026},
      eprint={2605.08454},
      archivePrefix={arXiv},
      primaryClass={cs.LG},
      url={https://arxiv.org/abs/2605.08454},

}

@inproceedings{wang2025transition,
  title     = {Transition Models: Rethinking the Generative Learning Objective},
  author    = {Wang, Zidong and Zhang, Yiyuan and Yue, Xiaoyu and Yue, Xiangyu and Li, Yangguang and Ouyang, Wanli and Bai, Lei},
  booktitle = {Proceedings of the IEEE/CVF Conference on Computer Vision and Pattern Recognition (CVPR)},
  year      = {2026},
  pages     = {29178--29189},
  url       = {https://openaccess.thecvf.com/content/CVPR2026/html/Wang_Transition_Models_Rethinking_the_Generative_Learning_Objective_CVPR_2026_paper.html}
}

@inproceedings{chen2018neural,
  author    = {Chen, Ricky T. Q. and Rubanova, Yulia and Bettencourt, Jesse and Duvenaud, David K.},
  title     = {Neural Ordinary Differential Equations},
  booktitle = {Advances in Neural Information Processing Systems},
  volume    = {31},
  pages     = {6572--6583},
  publisher = {Curran Associates, Inc.},
  year      = {2018},
  url       = {https://proceedings.neurips.cc/paper_files/paper/2018/file/69386f6bb1dfed68692a24c8686939b9-Paper.pdf}
}

@article{lam2023graphcast,
  author  = {Lam, Remi and Sanchez-Gonzalez, Alvaro and Willson, Matthew and
             Wirnsberger, Peter and Fortunato, Meire and Alet, Ferran and
             Ravuri, Suman and Ewalds, Timo and Eaton-Rosen, Zach and Hu, Weihua and
             Merose, Alexander and Hoyer, Stephan and Holland, George and
             Vinyals, Oriol and Stott, Jacklynn and Pritzel, Alexander and
             Mohamed, Shakir and Battaglia, Peter},
  title   = {Learning skillful medium-range global weather forecasting},
  journal = {Science},
  volume  = {382},
  number  = {6677},
  pages   = {1416--1421},
  year    = {2023},
  doi     = {10.1126/science.adi2336}
}

@misc{pathak2022fourcastnet,
      title={{FourCastNet}: A Global Data-driven High-resolution Weather Model using Adaptive {F}ourier Neural Operators},
      author={Jaideep Pathak and Shashank Subramanian and Peter Harrington and Sanjeev Raja and Ashesh Chattopadhyay and Morteza Mardani and Thorsten Kurth and David Hall and Zongyi Li and Kamyar Azizzadenesheli and Pedram Hassanzadeh and Karthik Kashinath and Animashree Anandkumar},
      year={2022},
      eprint={2202.11214},
      archivePrefix={arXiv},
      primaryClass={physics.ao-ph},
      url={https://arxiv.org/abs/2202.11214},

}

@article{kochkov2021machine,
  author  = {Kochkov, Dmitrii and Smith, Jamie A. and Alieva, Ayya and Wang, Qing and Brenner, Michael P. and Hoyer, Stephan},
  title   = {Machine learning--accelerated computational fluid dynamics},
  journal = {Proceedings of the National Academy of Sciences},
  volume  = {118},
  number  = {21},
  pages   = {e2101784118},
  year    = {2021},
  doi     = {10.1073/pnas.2101784118}
}

@article{bi2023pangu,
  author  = {Bi, Kaifeng and Xie, Lingxi and Zhang, Hengheng and Chen, Xin and Gu, Xiaotao and Tian, Qi},
  title   = {Accurate medium-range global weather forecasting with {3D} neural networks},
  journal = {Nature},
  volume  = {619},
  number  = {7970},
  pages   = {533--538},
  year    = {2023},
  doi     = {10.1038/s41586-023-06185-3}
}

@inproceedings{pfaff2021learning,
  author    = {Tobias Pfaff and Meire Fortunato and Alvaro Sanchez-Gonzalez and Peter W. Battaglia},
  title     = {Learning Mesh-Based Simulation with Graph Networks},
  booktitle = {International Conference on Learning Representations},
  year      = {2021},
  url       = {https://openreview.net/forum?id=roNqYL0\_XP}
}

@inproceedings{brandstetter2022message,
  title     = {Message Passing Neural {PDE} Solvers},
  author    = {Brandstetter, Johannes and Worrall, Daniel E. and Welling, Max},
  booktitle = {International Conference on Learning Representations},
  year      = {2022},
  url       = {https://openreview.net/forum?id=vSix3HPYKSU},
}

@article{koppe2019identifying,
      title={Identifying nonlinear dynamical systems via generative recurrent neural networks with applications to {fMRI}},
      author={Koppe, Georgia and Toutounji, Hazem and Kirsch, Peter and Lis, Stefanie and Durstewitz, Daniel},
      journal={{PLOS} Computational Biology},
      volume={15},
      number={8},
      pages={e1007263},
      year={2019},
      doi={10.1371/journal.pcbi.1007263},
      eprint={1902.07186},
      archivePrefix={arXiv},
  url = {https://arxiv.org/abs/1902.07186}
}

@inproceedings{hess2023generalized,
      title={Generalized Teacher Forcing for Learning Chaotic Dynamics},
      author={Hess, Florian and Monfared, Zahra and Brenner, Manuel and Durstewitz, Daniel},
      booktitle={Proceedings of the 40th International Conference on Machine Learning},
      series={Proceedings of Machine Learning Research},
      volume={202},
      pages={13017--13049},
      year={2023},
      publisher={PMLR},
      eprint={2306.04406},
      archivePrefix={arXiv},
      primaryClass={cs.LG},
  url = {https://arxiv.org/abs/2306.04406}
}

@inproceedings{mikhaeil2022difficulty,
      title={On the difficulty of learning chaotic dynamics with {RNNs}},
      author={Mikhaeil, Jonas M. and Monfared, Zahra and Durstewitz, Daniel},
      booktitle={Advances in Neural Information Processing Systems},
      volume={35},
      year={2022},
      eprint={2110.07238},
      archivePrefix={arXiv},
      primaryClass={cs.LG},
  url = {https://arxiv.org/abs/2110.07238}
}

@article{grassberger1983measuring,
  title={Measuring the strangeness of strange attractors},
  author={Grassberger, Peter and Procaccia, Itamar},
  journal={Physica D: Nonlinear Phenomena},
  volume={9},
  number={1-2},
  pages={189--208},
  year={1983},
  doi={10.1016/0167-2789(83)90298-1},
  url = {https://doi.org/10.1016/0167-2789(83)90298-1}
}

@article{datseris2023estimating,
  title={Estimating fractal dimensions: A comparative review and open source implementations},
  author={Datseris, George and Kottlarz, Inga and Braun, Anton P. and Parlitz, Ulrich},
  journal={Chaos: An Interdisciplinary Journal of Nonlinear Science},
  volume={33},
  number={10},
  pages={102101},
  year={2023},
  doi={10.1063/5.0160394},
  url = {https://arxiv.org/abs/2109.05937}
}

@article{kutta1901beitrag,
  title={Beitrag zur n{\"a}herungsweisen {I}ntegration totaler {D}ifferentialgleichungen},
  author={Kutta, Wilhelm},
  journal={Zeitschrift f{\"u}r Mathematik und Physik},
  volume={46},
  pages={435--453},
  year={1901}

}

@article{rossler1976equation,
  title={An equation for continuous chaos},
  author={R{\"o}ssler, Otto E.},
  journal={Physics Letters A},
  volume={57},
  number={5},
  pages={397--398},
  year={1976},
  url = {https://doi.org/10.1016/0375-9601(76)90101-8}
}

@article{chen1999yet,
  title={Yet another chaotic attractor},
  author={Chen, Guanrong and Ueta, Tetsushi},
  journal={International Journal of Bifurcation and Chaos},
  volume={9},
  number={7},
  pages={1465--1466},
  year={1999},
  url = {https://doi.org/10.1142/S0218127499001024}
}

@article{thomas1999deterministic,
  title={Deterministic chaos seen in terms of feedback circuits: Analysis, synthesis, ``labyrinth chaos''},
  author={Thomas, Ren{\'e}},
  journal={International Journal of Bifurcation and Chaos},
  volume={9},
  number={10},
  pages={1889--1905},
  year={1999},
  url = {https://doi.org/10.1142/S0218127499001383}
}

@book{sprott2003chaos,
  title={Chaos and Time-Series Analysis},
  author={Sprott, Julien Clinton},
  year={2003},
  publisher={Oxford University Press}

}

@article{sprott2007labyrinth,
  title={Labyrinth chaos},
  author={Sprott, Julien Clinton and Chlouverakis, Konstantinos E.},
  journal={International Journal of Bifurcation and Chaos},
  volume={17},
  number={6},
  pages={2097--2108},
  year={2007},
  url = {https://doi.org/10.1142/S0218127407018245}
}

@book{sprott2010elegant,
  title={Elegant Chaos: Algebraically Simple Chaotic Flows},
  author={Sprott, Julien Clinton},
  year={2010},
  publisher={World Scientific},
  url = {https://doi.org/10.1142/7183}
}

@article{eftekhari2012numerical,
  title={Numerical simulation of chaotic dynamical systems by the method of differential quadrature},
  author={Eftekhari, S. A. and Jafari, A. A.},
  journal={Scientia Iranica},
  volume={19},
  number={5},
  pages={1299--1315},
  year={2012},
  doi={10.1016/j.scient.2012.08.003},
  url = {https://doi.org/10.1016/j.scient.2012.08.003}
}

@article{noorani2007comparing,
  title={Comparing numerical methods for the solutions of the {C}hen system},
  author={Noorani, M. S. M. and Hashim, I. and Ahmad, R. and Bakar, S. A. and Ismail, E. S. and Zakaria, A. M.},
  journal={Chaos, Solitons \& Fractals},
  volume={32},
  number={4},
  pages={1296--1304},
  year={2007},
  publisher={Elsevier},
  url = {https://doi.org/10.1016/j.chaos.2005.12.036}
}

@inproceedings{kingma2015adam,
  title={Adam: A method for stochastic optimization},
  author={Kingma, Diederik P. and Ba, Jimmy},
  booktitle={International Conference on Learning Representations},
  year={2015},
  url = {https://arxiv.org/abs/1412.6980}
}

@inproceedings{golovin2017vizier,
  title={Google {V}izier: A service for black-box optimization},
  author={Golovin, Daniel and Solnik, Benjamin and Moitra, Subhodeep and Kochanski, Greg and Karro, John and Sculley, D.},
  booktitle={Proceedings of the 23rd ACM SIGKDD International Conference on Knowledge Discovery and Data Mining},
  pages={1487--1495},
  year={2017},
  url = {https://doi.org/10.1145/3097983.3098043}
}

@inproceedings{song2022vizier,
  title={Open source {V}izier: Distributed infrastructure and {API} for reliable and flexible blackbox optimization},
  author={Song, Xingyou and Perel, Sagi and Lee, Chansoo and Kochanski, Greg and Golovin, Daniel},
  booktitle={International Conference on Automated Machine Learning},
  year={2022},
  url = {https://arxiv.org/abs/2207.13676}
}

@inproceedings{peebles2023scalable,
  title={Scalable diffusion models with transformers},
  author={Peebles, William and Xie, Saining},
  booktitle={IEEE/CVF International Conference on Computer Vision},
  year={2023},
  url = {https://arxiv.org/abs/2212.09748}
}

@inproceedings{boffi2025consistency,
  author    = {Nicholas M. Boffi and Michael S. Albergo and Eric Vanden-Eijnden},
  title     = {How to build a consistency model: Learning flow maps via self-distillation},
  booktitle = {Advances in Neural Information Processing Systems},
  volume    = {38},
  publisher = {Curran Associates, Inc.},
  year      = {2025},
  doi       = {10.52202/085713-1121},
  url       = {https://proceedings.neurips.cc/paper_files/paper/2025/hash/2ff3605c9e442d55156d99fe4f5f228d-Abstract-Conference.html}
}

@inproceedings{durstewitz2026position,
  title     = {Position: Why a Dynamical Systems Perspective is Needed to Advance Time Series Modeling},
  author    = {Durstewitz, Daniel and Hemmer, Christoph J{\"u}rgen and Hess, Florian and Doll, Charlotte Ricarda and Eisenmann, Lukas},
  booktitle = {Proceedings of the 43rd International Conference on Machine Learning},
  year      = {2026},
  url       = {https://openreview.net/forum?id=h2R3VYWpBm},
}

@inproceedings{shikhman2026semigroup,
  author        = {Shikhman, Lennon J.},
  title         = {Semigroup Consistency as a Diagnostic for Learned Physics Simulators},
  booktitle     = {ICML 2026 Workshop on AI for Physics (AI4Physics)},
  year          = {2026},
  eprint        = {2605.26324},
  archivePrefix = {arXiv},
  primaryClass  = {cs.LG},
  url           = {https://openreview.net/forum?id=MeAFOZnrvM},
}

@article{raissi2019physics,
      title={Physics-informed neural networks: A deep learning framework for solving forward and inverse problems involving nonlinear partial differential equations},
      author={Raissi, Maziar and Perdikaris, Paris and Karniadakis, George E.},
      journal={Journal of Computational Physics},
      volume={378},
      pages={686--707},
      year={2019},
}

@inproceedings{frans2025shortcut,
  title         = {One Step Diffusion via Shortcut Models},
  author        = {Frans, Kevin and Hafner, Danijar and Levine, Sergey and Abbeel, Pieter},
  booktitle     = {International Conference on Learning Representations},
  year          = {2025},
  eprint        = {2410.12557},
  archivePrefix = {arXiv},
  url           = {https://arxiv.org/abs/2410.12557}
}

@article{qin2019data,
  title         = {Data driven governing equations approximation using deep neural networks},
  author        = {Qin, Tong and Wu, Kailiang and Xiu, Dongbin},
  journal       = {Journal of Computational Physics},
  volume        = {395},
  pages         = {620--635},
  year          = {2019},
  eprint        = {1811.05537},
  archivePrefix = {arXiv},
  url           = {https://arxiv.org/abs/1811.05537}
}

@article{chen2023deeposg,
  title         = {Deep-{OSG}: Deep Learning of Operators in Semigroup},
  author        = {Chen, Junfeng and Wu, Kailiang},
  journal       = {Journal of Computational Physics},
  volume        = {493},
  pages         = {112498},
  year          = {2023},
  eprint        = {2302.03358},
  archivePrefix = {arXiv},
  url           = {https://arxiv.org/abs/2302.03358}
}

@inproceedings{zhou2026terminal,
  title         = {Terminal Velocity Matching},
  author        = {Zhou, Linqi and Parger, Mathias and Haque, Ayaan and Song, Jiaming},
  booktitle     = {International Conference on Learning Representations},
  year          = {2026},
  eprint        = {2511.19797},
  archivePrefix = {arXiv},
  url           = {https://arxiv.org/abs/2511.19797}
}

\clearpage
\appendix
\restorespacing

\section{Method Cheatsheet}
\label{app:methods}
\cref{tab:method-summary} collects the definitions of all methods in \cref{tab:main-avg}.
For the evolution operators, write $h=\Delta t$ and $\psi_{\theta}(u,h)=u+h\bar{v}_{\theta}(u,h)$.
The model's own velocity at a state $u$ is $\bar{v}_{\theta}(u,0)$, and $D\bar{v}_{\theta}(u,h)[\cdot]$ denotes the Fr\'echet derivative with respect to the state, acting on a velocity.
The consistency conditions in the table are training constraints: we minimize the expected squared difference between their two sides, with the right-hand side held under stop-gradient.
These terms supplement the data MSE with weight $w(\Delta t/\dtmax)^p$, where $p=2$ and $w$ is tuned per system, as in \cref{eq:total,eq:loss-total}.
Expectations are over training states and horizons, and also over the split fraction $\alpha\in[0,1]$.

{\setlength{\intextsep}{2\baselineskip}
\begin{table}[!ht]
\caption{\textbf{Method definitions for the benchmark tables.}
The evolution-operator rows give the training objective or consistency condition; the baseline rows give the prediction rule.
NFE is the number of network evaluations used to reach five Lyapunov times in either evaluation regime.
The final column reports the average relative $L_2$ error at five Lyapunov times from \cref{tab:main-avg}.
Physics-distilled rows are shaded, and bold marks the lowest error within each group and evaluation regime.}
\label{tab:method-summary}
\centering
\small
\setlength{\tabcolsep}{1.5pt}
\renewcommand{\arraystretch}{1.18}
\begin{tabular*}{\textwidth}{@{}p{0.22\textwidth}p{\dimexpr0.605\textwidth-6\tabcolsep\relax}p{0.055\textwidth}p{0.12\textwidth}@{}}
\toprule
Method & Equation and distinguishing feature & NFE & Avg.\ $\rell\!\downarrow$\newline ID / OOD \\
\midrule
\rowcolor{black!8}\multicolumn{4}{l}{\emph{Evolution operator: elapsed-time step $h$}} \\
Data & $\mathcal{L}=\mathbb{E}\!\left[\big\|\bar{v}_{\theta}(u_t,h)-\tfrac{u_{t+h}-u_t}{h}\big\|^2\right]$.\newline
Paired-trajectory MSE ($h>0$). & 100 & 0.46 / 0.44 \\
\addlinespace
Semigroup & $\psi_{\theta}(u,h)=\psi_{\theta}\big(\psi_{\theta}(u,\alpha h),(1-\alpha)h\big)$.\newline
One jump agrees with two composed jumps. & 100 & 0.46 / 0.44 \\
\addlinespace
Eulerian-self & $\bar{v}_{\theta}(u,h)=\bar{v}_{\theta}(u,0)+hD\bar{v}_{\theta}(u,h)[\bar{v}_{\theta}(u,0)]-h\partial_h\bar{v}_{\theta}(u,h)$.\newline
Model velocity at the starting state $u$. & 100 & 0.45 / 0.43 \\
\addlinespace
Lagrangian-self & $\bar{v}_{\theta}(u,h)=\bar{v}_{\theta}\big(u+h\bar{v}_{\theta}(u,h),0\big)-h\partial_h\bar{v}_{\theta}(u,h)$.\newline
Model velocity at the predicted endpoint. & 100 & 0.45 / 0.48 \\
\addlinespace
\rowcolor{physhl}Eulerian-phys & $\bar{v}_{\theta}(u,h)=\mathcal{F}(u)+hD\bar{v}_{\theta}(u,h)[\mathcal{F}(u)]-h\partial_h\bar{v}_{\theta}(u,h)$.\newline
Governing-equation velocity at the starting state $u$. & 100 & \textbf{0.28} / \textbf{0.27} \\
\addlinespace
\rowcolor{physhl}Lagrangian-phys & $\bar{v}_{\theta}(u,h)=\mathcal{F}\big(u+h\bar{v}_{\theta}(u,h)\big)-h\partial_h\bar{v}_{\theta}(u,h)$.\newline
Governing-equation velocity at the predicted endpoint. & 100 & 0.31 / 0.30 \\
\midrule
\rowcolor{black!8}\multicolumn{4}{l}{\emph{Baselines}} \\
AR $v$-pred & $u_{t+\delta t}=u_t+\delta t\,f_{\theta}(u_t)$.\newline
Velocity prediction; fixed one-frame step $\delta t$. & 1000 & \textbf{0.27} / \textbf{0.26} \\
\addlinespace
AR $x$-pred (cond.\ $t$) & $u_{t+\delta t}=f_{\theta}(u_t,t)$.\newline
Direct next-state prediction conditioned on absolute time. & 1000 & 0.64 / 1.2 \\
\addlinespace
AR $x$-pred (uncond.\ $t$) & $u_{t+\delta t}=f_{\theta}(u_t)$.\newline
Direct next-state prediction without time conditioning. & 1000 & 0.59 / 0.57 \\
\addlinespace
CFO & $\partial_tu_t=f_{\theta}(u_t,t)$, integrated with RK4.\newline
Continuous-time velocity conditioned on absolute time. & 4000 & 0.30 / 1.2 \\
\bottomrule
\end{tabular*}
\end{table}}

In the baseline rows, $f_{\theta}$ denotes the separately trained network of the corresponding baseline, which returns the next state or a velocity as its row shows.
In the baseline prediction rules, $u_t$ denotes the current rollout state.
The Eulerian and Lagrangian equations differ in where the velocity is evaluated and whether a derivative with respect to the state is required; within either form, self and physics variants differ only in the velocity source.
Their full expressions and stop-gradient placement appear in \cref{eq:sg-eulerian,eq:sg-lagrangian,eq:sg-semigroup}, and training uses the normalized physical field in \cref{eq:physics-normalized}.
The rollout step counts and CFO's four network evaluations per RK4 step are specified in \cref{app:metrics}.

\clearpage

\FloatBarrier
\section{Proofs and learning objectives}
\label{sec:theory-consistency}
To make the time-invariant formulation precise, we work in the autonomous setting throughout this appendix and define every symbol used in the consistency equations.

\subsection{Consistency Equations}
\label{subsec:consistency-proof}

\paragraph{Symbol definitions.}
We use the following notation throughout this appendix.
\begin{itemize}
    \item $\mathcal{H}$ is the Hilbert space of solution states, and $\mathcal{U} \subset \mathcal{H}$ is the set of admissible states on which the dynamics are defined.
    \item $t \in [0,T]$ is the physical time, and $\Delta t \in [0,\dtmax]$ is the elapsed prediction horizon.
    \item $u_t \in \mathcal{U}$ denotes the exact solution state at time $t$.
    \item $v : \mathcal{U} \to \mathcal{H}$ is the autonomous velocity field, and the exact trajectory satisfies $\partial_t u_t = v(u_t)$.
    \item $\psi : \mathcal{U} \times [0,\dtmax] \to \mathcal{U}$ is the time-invariant evolution operator, where $\psi(u_t,\Delta t)$ predicts the state after elapsed time $\Delta t$.
    \item $D\psi(u,\Delta t) : \mathcal{H} \to \mathcal{H}$ denotes the Fr\'echet derivative of the map $u \mapsto \psi(u,\Delta t)$ with respect to the state variable, and $D\psi(u,\Delta t)[h]$ is its action on a direction $h \in \mathcal{H}$.
    \item $\partial_{\Delta t}\psi(u,\Delta t)$ denotes the partial derivative of $\psi$ with respect to the elapsed-time argument.
\end{itemize}

\paragraph{Assumptions.}
Assume that the trajectory $u : [0,T] \to \mathcal{U}$, the velocity field $v$, and the operator $\psi$ are continuously differentiable, that the operator does not change the state over zero elapsed time, $\psi(u,0) = u$, and that it satisfies the time-shift invariance along one trajectory,
\begin{equation}
\psi(u_t,\Delta t) = \psi(u_{t+s},\Delta t-s),
\label{eq:time-shift-invariance}
\end{equation}
for every shift $s \in [0,\Delta t]$.
So that every state the proofs below evaluate lies on the trajectory, we take the base time in $t \in [0,\,T-\dtmax]$ and the horizon in $\Delta t \in [0,\dtmax]$, which places $u_{t+\Delta t}$ inside $[0,T]$ for every horizon we consider.

\begin{proposition}[Consistency equations]
\label{prop:consistency-equations}
Under the assumptions above, the time-invariant evolution operator satisfies, for every $u_t$ and every $\Delta t\in[0,\dtmax]$:
\begin{itemize}
\item[(i)] the \emph{Eulerian consistency equation}
\begin{equation}
\partial_{\Delta t} \psi(u_t, \Delta t) = D \psi(u_t, \Delta t)[v(u_t)];
\label{eq:eulerian-consistency}
\end{equation}
\item[(ii)] the \emph{Lagrangian consistency equation}
\begin{equation}
\partial_{\Delta t}\psi(u_t,\Delta t)=v(\psi(u_t,\Delta t)).
\label{eq:lagrangian-consistency}
\end{equation}
\end{itemize}
\end{proposition}

\begin{proof}[Proof of \cref{eq:eulerian-consistency}]
Fix $t$ and $\Delta t$, and define
\begin{equation}
g(s) \triangleq \psi(u_{t+s}, \Delta t-s), \qquad s \in [0,\Delta t].
\end{equation}
By \cref{eq:time-shift-invariance}, $g$ is constant in $s$: every pair $(u_{t+s},\Delta t-s)$ lands on the same endpoint, so $g'(0) = 0$.
Applying the chain rule gives
\begin{equation}
g'(0)
= D\psi(u_t,\Delta t)\left[\frac{d}{ds}u_{t+s}\Big|_{s=0}\right]
+ \partial_{\Delta t}\psi(u_t,\Delta t)\frac{d}{ds}(\Delta t-s)\Big|_{s=0}.
\end{equation}
Using $\frac{d}{ds}u_{t+s}\vert_{s=0} = \partial_t u_t = v(u_t)$ and $\frac{d}{ds}(\Delta t-s)\vert_{s=0} = -1$, we obtain
\begin{equation}
0 = g'(0) = D\psi(u_t,\Delta t)[v(u_t)] - \partial_{\Delta t}\psi(u_t,\Delta t).
\end{equation}
Rearranging proves \cref{eq:eulerian-consistency}.
\end{proof}

\begin{proof}[Proof of \cref{eq:lagrangian-consistency}]
Taking $s = \Delta t$ in \cref{eq:time-shift-invariance} and using $\psi(u,0) = u$ gives
\begin{equation}
\psi(u_t,\Delta t) = \psi(u_{t+\Delta t}, 0) = u_{t+\Delta t},
\end{equation}
so the operator advances the exact trajectory.
Differentiating in $\Delta t$ and using the dynamics $\partial_t u_t = v(u_t)$ gives
\begin{equation}
\partial_{\Delta t}\psi(u_t,\Delta t)
= \frac{d}{d\Delta t}u_{t+\Delta t}
= v(u_{t+\Delta t})
= v(\psi(u_t,\Delta t)),
\end{equation}
which is \cref{eq:lagrangian-consistency}.
\end{proof}

\subsection{Differential Equations}
\label{subsec:flow-map-equations}
Suppose the evolution operator admits the averaged-velocity parameterization
\begin{equation}
\psi(u_t,\Delta t) = u_t + \Delta t\,\bar{v}(u_t,\Delta t).
\end{equation}
Here $\bar{v}(u_t,\Delta t)$ plays the role of an averaged velocity over the horizon $\Delta t$, and we assume it is jointly continuously differentiable in the state and the horizon.
Without that regularity the parameterization leaves $\bar{v}(u_t,0)$ free, since the factor $\Delta t$ annihilates it at zero horizon, and \cref{lem:boundary-condition} can fail.

\begin{lemma}[Boundary condition]
\label{lem:boundary-condition}
The averaged velocity is anchored to the instantaneous velocity at zero horizon,
\begin{equation}
\bar{v}(u_t,0)=v(u_t).
\label{eq:lagrangian-boundary}
\end{equation}
\end{lemma}

\begin{proof}
Differentiating the parameterization in $\Delta t$ and evaluating at $\Delta t=0$ gives
\begin{equation}
\partial_{\Delta t}\psi(u_t,\Delta t)\big|_{\Delta t=0}
=\Big[\bar{v}(u_t,\Delta t)+\Delta t\,\partial_{\Delta t}\bar{v}(u_t,\Delta t)\Big]_{\Delta t=0}
=\bar{v}(u_t,0).
\end{equation}
Evaluating \cref{eq:lagrangian-consistency} at $\Delta t=0$ gives
\begin{equation}
\partial_{\Delta t}\psi(u_t,\Delta t)\big|_{\Delta t=0}
=v\big(\psi(u_t,0)\big)
=v(u_t).
\end{equation}
Equating the two proves \cref{eq:lagrangian-boundary}.
\end{proof}

\begin{proposition}[Differential consistency equations]
\label{prop:differential-equations}
Under \cref{eq:eulerian-consistency,eq:lagrangian-consistency} respectively, the averaged velocity satisfies:
\begin{itemize}
\item[(i)] the \emph{Eulerian equation}
\begin{equation}
\bar{v}(u_t,\Delta t)
= v(u_t) + \Delta t \left( D\bar{v}(u_t,\Delta t)[v(u_t)] - \partial_{\Delta t}\bar{v}(u_t,\Delta t) \right);
\label{eq:theory-pimf-equation}
\end{equation}
\item[(ii)] the \emph{Lagrangian equation}
\begin{equation}
\bar{v}(u_t,\Delta t)
= v\!\left(u_t+\Delta t\,\bar{v}(u_t,\Delta t)\right)
- \Delta t\,\partial_{\Delta t}\bar{v}(u_t,\Delta t).
\label{eq:theory-pimf-equation-lagrangian}
\end{equation}
\end{itemize}
\end{proposition}

\begin{proof}[Proof of \cref{eq:theory-pimf-equation}]
The product rule gives
\begin{equation}
\partial_{\Delta t}\psi(u_t,\Delta t)
= \bar{v}(u_t,\Delta t) + \Delta t\,\partial_{\Delta t}\bar{v}(u_t,\Delta t).
\label{eq:partial-psi-expanded}
\end{equation}
Since $\psi(\cdot,\Delta t)$ is the identity plus $\Delta t\,\bar{v}(\cdot,\Delta t)$, linearity of the derivative gives
\begin{equation}
D\psi(u_t,\Delta t)[v(u_t)]
= v(u_t) + \Delta t\,D\bar{v}(u_t,\Delta t)[v(u_t)].
\label{eq:Dpsi-expanded}
\end{equation}
Substituting \cref{eq:partial-psi-expanded,eq:Dpsi-expanded} into \cref{eq:eulerian-consistency} gives
\begin{equation}
\bar{v}(u_t,\Delta t) + \Delta t\,\partial_{\Delta t}\bar{v}(u_t,\Delta t)
= v(u_t) + \Delta t\,D\bar{v}(u_t,\Delta t)[v(u_t)].
\end{equation}
Rearranging proves \cref{eq:theory-pimf-equation}.
\end{proof}

\begin{proof}[Proof of \cref{eq:theory-pimf-equation-lagrangian}]
The left-hand side of \cref{eq:lagrangian-consistency} is \cref{eq:partial-psi-expanded}.
Substituting the parameterization into its right-hand side gives
\begin{equation}
v(\psi(u_t,\Delta t)) = v\!\left(u_t+\Delta t\,\bar{v}(u_t,\Delta t)\right).
\label{eq:lagrangian-endpoint-velocity}
\end{equation}
Therefore \cref{eq:lagrangian-consistency} reads
\begin{equation}
\bar{v}(u_t,\Delta t)+\Delta t\,\partial_{\Delta t}\bar{v}(u_t,\Delta t)
= v\!\left(u_t+\Delta t\,\bar{v}(u_t,\Delta t)\right),
\end{equation}
and moving the $\Delta t\,\partial_{\Delta t}\bar{v}(u_t,\Delta t)$ term to the right-hand side proves \cref{eq:theory-pimf-equation-lagrangian}.
\end{proof}

\subsection{Semigroup Property}
\label{subsec:integral-consistency}

The semigroup property composes the horizon from two shorter jumps.
It follows directly from \cref{eq:time-shift-invariance} and uses no differential structure for learning.

\begin{proposition}[Semigroup property]
\label{prop:semigroup}
For $\Delta t_1,\Delta t_2\ge 0$ with $\Delta t_1+\Delta t_2\le\dtmax$,
\begin{equation}
\psi(u_t,\Delta t_1+\Delta t_2)=\psi\big(\psi(u_t,\Delta t_1),\Delta t_2\big).
\label{eq:semigroup-property}
\end{equation}
\end{proposition}

\begin{proof}
As in the proof of \cref{eq:lagrangian-consistency}, the operator advances the exact trajectory, $\psi(u_t,\Delta t_1)=u_{t+\Delta t_1}$.
Taking $s=\Delta t_1$ in \cref{eq:time-shift-invariance} with horizon $\Delta t_1+\Delta t_2$ gives
\begin{equation}
\psi(u_t,\Delta t_1+\Delta t_2)
=\psi\big(u_{t+\Delta t_1},\Delta t_2\big)
=\psi\big(\psi(u_t,\Delta t_1),\Delta t_2\big),
\end{equation}
which is \cref{eq:semigroup-property}.
\end{proof}

\subsection{Learning Objectives}
\label{subsec:learning-objectives}

Following the parameterization in \cref{subsec:flow-map-equations}, the model is $\psi_{\theta}(u,\Delta t)\triangleq u+\Delta t\,\bar{v}_{\theta}(u,\Delta t)$, and each equation yields a learning objective by squaring its residual.
Because the data are trajectories, we first use a regression loss on the averaged velocity as the data term,
\begin{equation}
\mathcal{L}_{\mathrm{data}}(\theta)
= \mathbb{E}_{u_t,\,\Delta t>0}\Big[\big\|\bar{v}_{\theta}(u_t,\Delta t)-\tfrac{u_{t+\Delta t}-u_t}{\Delta t}\big\|^2\Big].
\label{eq:loss-data}
\end{equation}

For the consistency losses, the instantaneous velocity $v$ is supplied by the model's own boundary value $\bar{v}_{\theta}(\cdot,0)$ (self-distillation, \cref{lem:boundary-condition}), and the Semigroup objective uses a fraction $\alpha$ to randomize the split.
The losses take expectations over states $u_t$ and horizons $\Delta t$.
The sampling distribution of $\Delta t$ is compared in \cref{app:regimes}, and other implementation choices in \cref{app:details}.
\begin{align}
\mathcal{L}_{\mathrm{E}}(\theta)
&= \mathbb{E}_{u_t,\Delta t}\Big[\big\|v(u_t)-\bar{v}_{\theta}(u_t,\Delta t)
-\Delta t\big(\partial_{\Delta t}\bar{v}_{\theta}(u_t,\Delta t)-D\bar{v}_{\theta}(u_t,\Delta t)[v(u_t)]\big)\big\|^2\Big],
\label{eq:loss-eulerian}\\
\mathcal{L}_{\mathrm{L}}(\theta)
&= \mathbb{E}_{u_t,\Delta t}\Big[\big\|v\big(\psi_{\theta}(u_t,\Delta t)\big)-\bar{v}_{\theta}(u_t,\Delta t)
-\Delta t\,\partial_{\Delta t}\bar{v}_{\theta}(u_t,\Delta t)\big\|^2\Big],
\label{eq:loss-lagrangian}\\
\mathcal{L}_{\mathrm{S}}(\theta)
&= \mathbb{E}_{u_t,\Delta t,\alpha}\Big[\big\|\psi_{\theta}(u_t,\Delta t)-\psi_{\theta}\big(\psi_{\theta}(u_t,\alpha\Delta t),(1-\alpha)\Delta t\big)\big\|^2\Big].
\label{eq:loss-semigroup}
\end{align}

We also weight the consistency term with the time-dependent coefficient $w\,(\Delta t/\dtmax)^{p}$:
\begin{equation}
\mathcal{L}(\theta)=\mathcal{L}_{\mathrm{data}}(\theta)+w\,\mathcal{L}^{(p)}_{\mathrm{cons}}(\theta),
\qquad w\ge 0,
\label{eq:loss-total}
\end{equation}
where $\mathcal{L}^{(p)}_{\mathrm{cons}}$ is one of \cref{eq:loss-eulerian,eq:loss-lagrangian,eq:loss-semigroup} with $(\Delta t/\dtmax)^{p}$ inside the expectation.
The comparison over $p$ is deferred to \cref{app:ablation-p}.
Here, $p=0$ means a constant weight; we perform hyperparameter tuning for $w$, and $w=0$ is the data-only control.

\paragraph{Stop-gradient placement.}
Let $\mathrm{sg}[\cdot]$ denote the stop-gradient operator, the identity in the forward pass with zero
derivative in the backward pass.
Every consistency target is built from the model's own predictions, so a gradient taken through the
target would let each objective train its own output.
For stability, the whole target is therefore frozen under a single stop-gradient, and each objective
regresses the direct evaluation onto it.
As implemented,
\cref{eq:loss-eulerian,eq:loss-lagrangian,eq:loss-semigroup}
carry their stop-gradients as
\begin{align}
\mathcal{L}_{\mathrm{E}}(\theta)
&= \mathbb{E}_{u_t,\Delta t}\Big[\big\|\bar{v}_{\theta}(u_t,\Delta t)
-\mathrm{sg}\big[v(u_t)+\Delta t\big(D\bar{v}_{\theta}(u_t,\Delta t)[v(u_t)]-\partial_{\Delta t}\bar{v}_{\theta}(u_t,\Delta t)\big)\big]\big\|^2\Big],
\label{eq:sg-eulerian}\\
\mathcal{L}_{\mathrm{L}}(\theta)
&= \mathbb{E}_{u_t,\Delta t}\Big[\big\|\bar{v}_{\theta}(u_t,\Delta t)
-\mathrm{sg}\big[v\big(u_t+\Delta t\,\bar{v}_{\theta}(u_t,\Delta t)\big)-\Delta t\,\partial_{\Delta t}\bar{v}_{\theta}(u_t,\Delta t)\big]\big\|^2\Big],
\label{eq:sg-lagrangian}\\
\mathcal{L}_{\mathrm{S}}(\theta)
&= \mathbb{E}_{u_t,\Delta t,\alpha}\Big[\big\|\psi_{\theta}(u_t,\Delta t)-\mathrm{sg}\big[\psi_{\theta}\big(\psi_{\theta}(u_t,\alpha\Delta t),(1-\alpha)\Delta t\big)\big]\big\|^2\Big].
\label{eq:sg-semigroup}
\end{align}

\subsection{Physics-Distillation}
\label{subsec:implementation-details}

Instead of using the model's own boundary value in \cref{subsec:learning-objectives}, we can use the governing equations $\mathcal{F}$ (physics-distillation) to supply the instantaneous velocity,
\begin{equation}
\partial_t u_t = v(u_t) = \mathcal{F}(u_t).
\label{eq:physics-velocity}
\end{equation}

\paragraph{Normalized coordinates.}
Let $C$ be the number of state channels, let $\mu\in\mathbb{R}^{C}$ and $\sigma\in\mathbb{R}^{C}$ be the
per-channel mean and standard deviation of the states in the training set, and let
$\operatorname{diag}(\sigma)$ be the diagonal matrix with entries $\sigma$ along the channel axis.
We train entirely in normalized coordinates, where states are z-scored,
$z=\operatorname{diag}(\sigma)^{-1}(u-\mu)$, and time is rescaled to the unit interval, so a trajectory of
$K$ snapshots (indexed from $0$ to $K-1$) with sampling time interval $\delta t$ spans the normalized time $\tau\in[0,1]$, with physical time $t=\tau\,T_{\mathrm{span}}$, where
$T_{\mathrm{span}}\triangleq (K-1)\,\delta t$.

Every objective of \cref{subsec:learning-objectives} is evaluated in these coordinates, and every reported
metric is computed in physical units after inverting the normalization.
Following this normalization, the governing equations $\mathcal{F}$ use the chain rule for physics-distillation,
\begin{equation}
\partial_{\tau} z
= T_{\mathrm{span}}\,\operatorname{diag}(\sigma)^{-1}\,\partial_t u
= T_{\mathrm{span}}\,\operatorname{diag}(\sigma)^{-1}\,\mathcal{F}\big(\mu+\operatorname{diag}(\sigma)\,z\big),
\end{equation}
Thus for implementation, model training uses the normalized governing equations
\begin{equation}
\mathcal{F}_{\mathrm{norm}}(z)
\triangleq T_{\mathrm{span}}\,\operatorname{diag}(\sigma)^{-1}\,\mathcal{F}\big(\mu+\operatorname{diag}(\sigma)\,z\big).
\label{eq:physics-normalized}
\end{equation}

\section{Datasets}
\label{app:datasets}
Each dataset holds $10{,}000$ trajectories of one chaotic system, stored as tensors of $1001$ frames in a three-dimensional state space and split into $9000$ training, $500$ validation, and $500$ test trajectories.
We generate all five datasets using the official implementation of dysts \citep{gilpin2021chaos}. Since we target long-term behavior, every trajectory starts on the attractor, drawn by the on-attractor initial-condition sampler of dysts: it integrates one long reference orbit and takes the starting points from it, so the dataset is a set of overlapping windows of that orbit and inherits its invariant measure. We space the starting points by a fixed interval along the reference, since random draws would repeat some of them. Every window is a segment of the same autonomous flow, so absolute time carries no information about what comes next, and the data determines exactly the time-invariant operator of \cref{sec:method}. From each starting point, DOP853~\citep{hairer1993solving} then integrates the trajectory, keeping the estimated relative error of every step below $10^{-10}$.
We run each trajectory for five Lyapunov times, so consecutive frames are $0.005$ Lyapunov times apart; see \cref{subsec:implementation-details} for the normalization.
We split the trajectories in their order along the reference, so every test trajectory starts more than $155\Tlyap$ after the last training trajectory ends, with the validation trajectories in between.

\paragraph{Lyapunov exponents and times.}
A deterministic system is chaotic when its long-term behavior is bounded, aperiodic, and sensitive to initial conditions~\citep{sprott2003chaos}.
The sensitivity is quantified by the largest Lyapunov exponent: start two trajectories from a state $u_0$ and its perturbation $u_0 + \delta u_0$, write $\delta u_t$ for their difference at time $t$, and define, for almost every initial state on the attractor,
\begin{equation}
\lambda_{\max} = \lim_{t \to \infty} \lim_{\|\delta u_0\| \to 0} \frac{1}{t} \ln \frac{\|\delta u_t\|}{\|\delta u_0\|},
\label{eq:lyapunov-exponent}
\end{equation}
and sensitivity corresponds to $\lambda_{\max} > 0$: nearby states diverge exponentially fast.
Its inverse, the Lyapunov time $\Tlyap \triangleq 1/\lambda_{\max}$, is the time over which a small error grows by a factor of $e$, so it is the natural time unit for comparing systems that run at different speeds.
Because the horizon and every metric are denominated in Lyapunov times, $\lambda_{\max}$ is pinned to published values at exactly the constants we integrate, listed in \cref{tab:systems} together with the governing equations of each system.

\begin{table}[h]
\caption{\textbf{The five systems.}
Parameters, the largest Lyapunov exponent $\lambda_{\max}$, and its source.}
\label{tab:systems}
\centering
\small
\begin{tabular}{llll}
\toprule
System & Parameters & $\lambda_{\max}$ & Source \\
\midrule
Lorenz & $(\sigma, \rho, \beta) = (10,\, 28,\, 8/3)$ & $0.9056$ & \citet{sprott2003chaos} \\
Chen & $(a, b, c) = (35,\, 3,\, 28)$ & $2.0272$ & \citet{noorani2007comparing, eftekhari2012numerical} \\
R\"ossler & $(a, b, c) = (0.2,\, 0.2,\, 5.7)$ & $0.0714$ & \citet{sprott2003chaos} \\
Thomas & $b = 0.1$ & $0.055$ & \citet{sprott2007labyrinth} \\
Halvorsen & $a = 1.3$ & $0.6928$ & \citet{sprott2010elegant} \\
\bottomrule
\end{tabular}
\end{table}

\paragraph{Lorenz.}
The classical model of atmospheric convection, with its two-winged attractor~\citep{lorenz1963deterministic},
\begin{equation}
\left\{
\begin{gathered}
\dot{x} = \sigma (y - x), \\
\dot{y} = x (\rho - z) - y, \\
\dot{z} = x y - \beta z.
\end{gathered}
\right.
\end{equation}

\paragraph{Chen.}
A close relative of Lorenz with a stronger twist, chaotic more than twice as fast \citep{chen1999yet},
\begin{equation}
\left\{
\begin{gathered}
\dot{x} = a (y - x), \\
\dot{y} = (c - a) x - x z + c y, \\
\dot{z} = x y - b z.
\end{gathered}
\right.
\end{equation}

\paragraph{R\"ossler.}
The simplest spiral-type chaotic flow, with the single quadratic term $zx$~\citep{rossler1976equation},
\begin{equation}
\left\{
\begin{gathered}
\dot{x} = -y - z, \\
\dot{y} = x + a y, \\
\dot{z} = b + z (x - c).
\end{gathered}
\right.
\end{equation}

\paragraph{Thomas.}
A cyclically symmetric flow with sinusoidal feedback, known for labyrinth chaos~\citep{thomas1999deterministic},
\begin{equation}
\left\{
\begin{gathered}
\dot{x} = \sin y - b x, \\
\dot{y} = \sin z - b y, \\
\dot{z} = \sin x - b z.
\end{gathered}
\right.
\end{equation}

\paragraph{Halvorsen.}
A cyclically symmetric quadratic flow with a single strange attractor~\citep{sprott2010elegant},
\begin{equation}
\left\{
\begin{gathered}
\dot{x} = -a x - 4 y - 4 z - y^2, \\
\dot{y} = -a y - 4 z - 4 x - z^2, \\
\dot{z} = -a z - 4 x - 4 y - x^2.
\end{gathered}
\right.
\end{equation}

\FloatBarrier
\section{Evaluation metrics}
\label{app:metrics}
Here, we define each metric used in our experiments mathematically for short-term and long-term evaluation, and provide some implementation details for the long-term Markovian metrics.

\subsection{Evaluation setup}
\label{app:metrics-notation}

We compute every metric on the test set of $500$ held-out trajectories, each with $1001$ frames in a $d$-dimensional state space, measured in physical units.
In our chaotic-systems setting, we compare $d = 3$ dimensional systems with different Lyapunov time scales, with details in \cref{app:datasets}.
Here, time is indexed $0, \dots, 1000$ and spans exactly five Lyapunov times for each system, so frame $200$ sits at $1\Tlyap$ and frame $1000$ at $5\Tlyap$ on every system.
A hat marks a quantity produced by the model (prediction $\hat u$, power spectrum $\hat p$); the same symbol without a hat is its ground-truth counterpart.

\subsection{Relative $L_2$ error}
\label{app:metrics-relL2}

For short-term behavior, we first use the relative $L_2$ error as the pointwise metric, which follows the classical convention \citep{li2020fourier, kovachki2023neural, hou2025cfolearningcontinuoustimepde}.
For the $500$ test trajectories, we write $u_{i,j} \in \mathbb{R}^{d}$ for frame $j$ of the $i$-th test trajectory, and then we define the $k$-call prediction of the horizon $H$ (in frames) for the $i$-th trajectory by the feedback recursion
\begin{equation}
\hat u^{(0)}_i = u_{i,0},
\qquad
\hat u^{(n)}_i = M_\theta\big(\hat u^{(n-1)}_i,\, H/k\big),
\quad n = 1, \dots, k,
\label{eq:app-kstep}
\end{equation}
where $M_\theta$ represents the model's output under step size $H/k$, and we report the average over the per-trajectory relative $L_2$ errors,
\begin{equation}
\rell(H) \;=\; \frac{1}{500} \sum_{i=1}^{500}
\frac{\big\| \hat u^{(k)}_i - u_{i,H} \big\|_2}{\big\| u_{i,H} \big\|_2 + \varepsilon},
\qquad
\varepsilon = 10^{-8}.
\label{eq:app-rell}
\end{equation}
We report $H = 200$ (one Lyapunov time, \cref{app:tables}) and $H = 1000$ (five Lyapunov times), and each method uses the following $k$: the evolution operator at $k = 20$ and $k = 100$, AR at $k = 200$ and $k = 1000$, and CFO at the same $k$ as AR, with four times the evaluations from its RK4 solver \citep{kutta1901beitrag}.
We measure at the horizon rather than aggregating over the whole trajectory, as CFO \citep{hou2025cfolearningcontinuoustimepde} does.

\subsection{Long-term statistics}
\label{app:metrics-longrun}

In chaotic systems, nearby trajectories separate exponentially, so no prediction stays aligned with the true trajectory beyond a few Lyapunov times \citep{mikhaeil2022difficulty}.
The meaningful long-term target is instead the statistics of the motion on the attractor, so we also evaluate the model's output in distributional and transition aspects \citep{koppe2019identifying, gilpin2021chaos, hess2023generalized, wang2024closure, zhang2025zeroshot}.
From the $500$ test initial states, each model rolls $50\Tlyap$: AR and CFO roll $10{,}000$ native frames and keep every tenth, and the evolution operator rolls $1000$ native steps of $\dtmax = 0.05\Tlyap$ ($10$ frames), so every model lands on the same $500 \times 1000 \times 3$ array sampled every $10$ frames, while the evolution operator runs $10$ times cheaper than AR.
For the ground truth, the reference is the high-accuracy solution (DOP853 \citep{hairer1993solving}) at the same times, and the first $5\Tlyap$ are discarded; the remainder is flattened over trajectories and times into two clouds of $450{,}000$ states each, one from the model's runs and one from the reference's. The two clouds capture the model's and the truth's long-term behavior on the attractor.

\newpage
\subsection{Correlation dimension}
\label{app:metrics-corrdim}

The attractor is a fractal, and its correlation dimension measures how the number of neighboring state pairs grows with distance \citep{grassberger1983measuring}.
We write $u_n \in \mathbb{R}^{d}$ for the $n$-th state of the flattened reference cloud (\cref{app:metrics-longrun}), $N$ for the number of states, and $\#$ for the size of a set, and then the correlation sum
\begin{equation}
C(r) = \frac{\#\{(n, n') : n \neq n',\; \| u_n - u_{n'} \|_2 \le r\}}{N (N - 1)}
\label{eq:app-corrsum}
\end{equation}
scales as $r^{D_2}$ at small $r$, and the exponent $D_2$ is the correlation dimension.
We estimate it by subsampling each cloud to $10{,}000$ states and fitting the slope of $\log C(r)$ against $\log r$ by least squares \citep{datseris2023estimating}.
We report $\edtwo = |\hat D_2 - D_2|$, where $D_2$ is estimated from the reference cloud and $\hat D_2$ from the model's cloud.
The difference catches a model whose runs blur the attractor's fine structure into a thicker set, even when every state stays close to the attractor.

\subsection{State-space divergence}
\label{app:metrics-kl}

For the distributional distance, we use $D_{\mathrm{stsp}}$, which asks whether the model's states are distributed over the state space like the truth's \citep{koppe2019identifying}.
For each trajectory, we place a unit-variance Gaussian on every state, so the trajectory becomes a density over the state space.
$D_{\mathrm{stsp}}$ is the Kullback--Leibler (KL) divergence from the reference density $p$ to the model density $\hat p$,
\begin{equation}
D_{\mathrm{stsp}} = \int p(x) \log \frac{p(x)}{\hat p(x)} \, dx.
\label{eq:app-dstsp}
\end{equation}
The integral is estimated by Monte Carlo sampling \citep{hess2023generalized}, using the implementation of the dysts benchmark \citep{gilpin2021chaos}, and averaged over a fixed subsample of $64$ trajectory pairs.
Because the densities live in the full state space, this also catches joint structure that per-coordinate comparisons miss.

\subsection{Power-spectrum distance}
\label{app:metrics-psd}

For the model's output, we also compare the long-term temporal behavior through its power spectrum.
For each coordinate $c$, every trajectory is standardized over time, and its spectrum is the squared magnitude of its discrete Fourier transform.
The spectra are averaged over the $500$ trajectories, smoothed with a Gaussian kernel over frequency bins \citep{hess2023generalized}, and normalized to sum to one, giving a distribution over frequencies $p_c$ for the reference and $\hat p_c$ for the model.
The distance is the Hellinger distance averaged over the coordinates,
\begin{equation}
D_{\mathrm{H}}
= \frac{1}{d} \sum_{c} \sqrt{\frac{1}{2} \sum_{f} \Big( \sqrt{\hat p_c(f)} - \sqrt{p_c(f)} \Big)^{2}}
\;\in\; [0, 1],
\label{eq:app-psd}
\end{equation}
with $f$ running over the frequency bins.
It is zero when the spectra coincide and one when they do not overlap, so a model oscillating at the wrong frequencies scores high even if it stays on the attractor.

\subsection{Transition-matrix error}
\label{app:metrics-markov}

The metrics above check where the model spends its time and at which frequencies it oscillates, but not how it moves between regions of the attractor, so we also compare the switching dynamics \citep{ulam1960collection, prinz2011markov}: for Lorenz the two regions are the two wings of the butterfly \citep{lorenz1963deterministic}, and for the other systems they are the two halves found by clustering.
We split the reference cloud into two clusters by $k$-means \citep{lloyd1982least} and label every state of both runs by its nearest cluster center; a state too far from both centers has escaped the attractor and is labeled $\mathrm{OFF}$.
With $z_{i,m} \in \{0, 1, \mathrm{OFF}\}$ the label of trajectory $i$ at step $m$, the transition matrix is
\begin{equation}
P_{ab} = \frac{\#\{(i,m) : z_{i,m} = a,\; z_{i,m+1} = b\}}{\#\{(i,m) : z_{i,m} = a,\; z_{i,m+1} \neq \mathrm{OFF}\}},
\label{eq:app-markov}
\end{equation}
the probability that the next step (one step $=$ $10$ frames) lands in region $b$ given the current region $a$.
We report the transition-matrix error
\begin{equation}
\etrans = \tfrac{1}{2} \big( |\hat P_{01} - P_{01}| + |\hat P_{10} - P_{10}| \big),
\label{eq:app-markov-err}
\end{equation}
with $P$ estimated from the reference runs and $\hat P$ from the model, which asks whether the model switches between the two regions at the right rate.
\cref{alg:markov} states the full computation.

\begin{algorithm}[t]
\caption{Transition-matrix error.}
\label{alg:markov}
\begin{lstlisting}
def transition_error(U, U_hat):
  # U:     truth free run,  [B, N, d]
  # U_hat: model free run,  [B, N, d], same time grid
\end{lstlisting}
\vspace{2pt}
\begin{lstlisting}[backgroundcolor=\color{physhl}]
  ## split the truth attractor into two regions
  Z = standardize(U.reshape(-1, d))
  C = kmeans(Z, k=2).centers           # [2, d]
  r = 3 * cluster_radii(Z, C).max()    # escape threshold
\end{lstlisting}
\vspace{2pt}
\begin{lstlisting}[backgroundcolor=\color{physhl}]
  ## label every state, count transitions
  def assign(X):                       # labels in {0, 1, OFF}
    z, dist = nearest(standardize(X), C)   # uses U's mean/std
    return where(dist > r, OFF, z)

  def transition(X):
    z = assign(X)                                  # [B, N]
    s, t = z[:, :-1], z[:, 1:]                     # consecutive steps
    return [[mean(t[(s == a) & (t != OFF)] == b) for b in (0, 1)] for a in (0, 1)]
\end{lstlisting}
\vspace{2pt}
\begin{lstlisting}[backgroundcolor=\color{physhl}]
  ## compare the model's matrix to the truth's
  P = transition(U)
  P_hat = transition(U_hat)
  return (abs(P_hat[0][1] - P[0][1]) + abs(P_hat[1][0] - P[1][0])) / 2
\end{lstlisting}
\end{algorithm}

\FloatBarrier
\section{Implementation details}
\label{app:details}
\paragraph{Model architecture.}
We compare all methods on the same MLP with residual connections, using four blocks of hidden width $128$, and $341{,}379$ parameters in total.
Each block applies a LayerNorm, a GELU activation, and a linear layer. The output of each block is divided by four so that the time derivative with respect to $\Delta t$ does not grow with depth.
For time conditioning, the step size $\Delta t$ enters every block through AdaLN-Zero \citep{peebles2023scalable} on the sinusoidal features $[\Delta t, \sin(2\pi f_k \Delta t), \cos(2\pi f_k \Delta t)]$, with four frequencies $f_k$ spaced linearly in $[1, 8]$, and a two-layer network with a zero-initialized output layer maps them to a learnable scale and shift.

\paragraph{Method details.}
For the AR and CFO baselines, the model either predicts the instantaneous or average velocity ($v$-pred), which is multiplied by the time step $\Delta t$, or the next state ($x$-pred) directly.
For the evolution operator, the network outputs the averaged velocity $\bar{v}_{\theta}(u, \Delta t)$ and the operator itself is $\psi_{\theta}(u, \Delta t) = u + \Delta t\,\bar{v}_{\theta}(u, \Delta t)$, so one forward pass covers the whole horizon.

\paragraph{Hyperparameter tuning.}
All tuning uses the official implementation of Vizier\footnote{\url{https://github.com/google/vizier}} \citep{golovin2017vizier, song2022vizier}, and selects hyperparameters by the relative L2 error on the validation set.
For each task, the data-only control first searches the learning rate and the EMA decay over $12$ or $16$ trials, then fixes the best pair for every consistency objective, for which we search only the weight $w$ over $6$ trials.
The AR and CFO baselines search the learning rate alone over $6$ trials.
All searches run on seed $0$, and seeds $1$ and $2$ retrain the selected cells unchanged.
For the out-of-distribution tasks, we reuse the values tuned on the corresponding in-distribution task without any further search.

\paragraph{Training.}
\cref{tab:training} lists the training hyperparameters shared by all methods, which follow the official implementation of CFO\footnote{\url{https://github.com/shannon-hou/CFO_official}} \citep{hou2025cfolearningcontinuoustimepde}.
The AR baselines train on the fixed one-frame step and CFO on absolute time, as in their original implementations.
For the evolution operator, each training batch draws states $u_t$ from the training trajectories and step sizes $\Delta t$, with details in \cref{app:regimes}.
The model selected for evaluation is the checkpoint with the best validation relative L2 error, checked every $5000$ iterations, using the EMA weights for the evolution operator.

\begin{table}[t]
\caption{\textbf{Training settings.} Fixed settings above, and the ranges searched by Vizier below.}
\label{tab:training}
\centering
\small
\begin{tabular}{ll}
\toprule
Hyperparameter & Value \\
\midrule
\rowcolor{black!8}\multicolumn{2}{l}{\emph{Training}} \\
Optimizer & Adam ($\beta_1 = 0.9$, $\beta_2 = 0.99$) \citep{kingma2015adam} \\
LR scheduler & constant \\
Batch size & $2048$ \\
Iterations & $200{,}000$ \\
Validation interval & $5000$ iterations \\
Random seeds & $0$, $1$, $2$ \\
\addlinespace
\rowcolor{black!8}\multicolumn{2}{l}{\emph{Vizier search space}} \\
Learning rate & $[10^{-5}, 10^{-3}]$, log scale \\
EMA decay (evolution operator) & $\log_{10}(1 - \mathrm{ema}) \in [-4, -1]$ \\
Consistency weight $w$ & $\log_{10} w \in [-2, 1.5]$ \\
\bottomrule
\end{tabular}
\end{table}

\FloatBarrier
\section{Ablation studies}
\label{app:ablations}
\paragraph{Sampling distribution and largest step ($\dtmax$).}
\label{app:regimes}
We follow Hamiltonian Flow Maps \citep{ripken2026hamiltonian} and draw the training step from one frame up to $\dtmax$ as $0.98\,\mathrm{Beta}(1, 2) + 0.02\,\mathrm{Uniform}$, which we call Beta, and we ablate that choice here against uniform and log-uniform at two values of $\dtmax$ on the Lorenz tasks, with the results in \cref{fig:regimes}.
The largest step matters more than the distribution.
At $\dtmax = 5\Tlyap$ every error is around $0.5$ or worse whichever distribution we draw from, and at $\dtmax = 0.05\Tlyap$ the distribution moves the error of the data-only model and the ones fed their own velocity by about a factor of two, while the physics-distilled ones change much less.
We leave for future work what kind of sampling distribution suits this setting best.

\begin{figure}[!b]
\centering
\includegraphics[width=\linewidth, trim=0 6.5 0 6.5, clip]{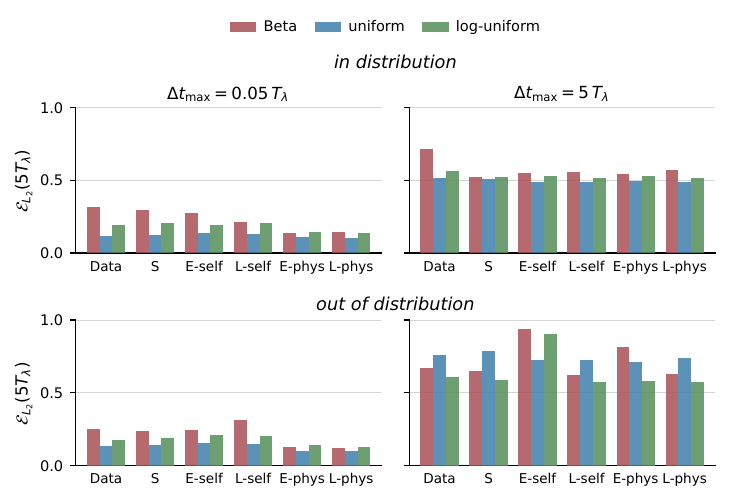}
\caption{\textbf{Sampling distribution and largest step.}
$\rell(5\Tlyap)$ on Lorenz at $p = 2$, averaged over three seeds, for the three distributions of the training step, and each panel pair splits by $\dtmax$.}
\label{fig:regimes}
\end{figure}

\paragraph{Adaptive weight coefficient ($p$).}
\label{app:ablation-p}
\cref{fig:p} sweeps the exponent of the graded consistency weight with the distribution fixed to Beta and $\dtmax = 0.05\Tlyap$.
The exponent moves the error far less than either of the other two settings, and the physics-distilled objectives lead at every value of it.

\begin{figure}[htbp]
\centering
\includegraphics[width=\linewidth, trim=0 6.5 0 6.5, clip]{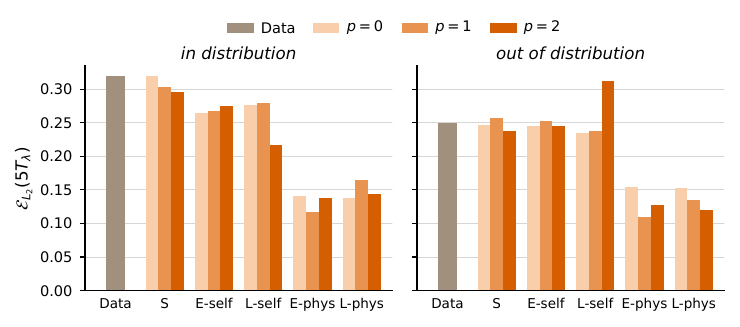}
\caption{\textbf{Adaptive weight coefficient.}
$\rell(5\Tlyap)$ on Lorenz under the Beta distribution at $\dtmax = 0.05\Tlyap$, averaged over three seeds.
The data-only model has no consistency term, so it is drawn once.}
\label{fig:p}
\end{figure}

\section{Additional results}
\label{app:tables}
\subsection{Wall-clock time}

In addition to measuring NFE, we also measure the wall-clock time to roll out five Lyapunov times on each of the five systems, with every method on the same network, listed in \cref{tab:wallclock}.
The evolution operators differ only in their training objective, so we time the data-only model for all of them.

\begin{table}[!h]
\caption{\textbf{Wall-clock time.}
Time to roll out one trajectory over five Lyapunov times on each of the five systems with a single NVIDIA RTX 4090, mean $\pm$ std over $100$ test trajectories, with each method at its native NFE.
Shading marks the evolution operator and bold the shortest time in each column.}
\label{tab:wallclock}
\centering
\small
\setlength{\tabcolsep}{4pt}
\begin{tabular}{lcccccc}
\toprule
 & & \multicolumn{5}{c}{Time (ms)} \\
\cmidrule(lr){3-7}
Method & NFE & Lorenz & Chen & R\"ossler & Thomas & Halvorsen \\
\midrule
AR $v$-pred & 1000 & 404.9 $\pm$ 1.5 & 396.9 $\pm$ 1.3 & 398.4 $\pm$ 1.7 & 401.0 $\pm$ 1.1 & 401.5 $\pm$ 1.1 \\
CFO & 4000 & 1668.7 $\pm$ 4.9 & 1636.4 $\pm$ 4.3 & 1642.2 $\pm$ 2.7 & 1650.3 $\pm$ 4.0 & 1665.5 $\pm$ 3.2 \\
\rowcolor{physhl}Evolution operator & 100 & \textbf{45.1 $\pm$ 0.1} & \textbf{44.3 $\pm$ 0.2} & \textbf{44.0 $\pm$ 0.1} & \textbf{45.0 $\pm$ 0.8} & \textbf{44.8 $\pm$ 0.2} \\
\bottomrule
\end{tabular}
\end{table}

\FloatBarrier
\subsection{Short horizon results}

\cref{tab:main-avg} reports the endpoint at five Lyapunov times, and by then every model has drifted far from the truth. Here we add the accuracy at one Lyapunov time, where they are all still close to it.

\begin{table}[!h]
\caption{\textbf{One-Lyapunov-time accuracy.}
$\rell(1\Tlyap)$, $\times 10^{-2}$, mean $\pm$ std over three seeds, with each method at its native step size, stated in the NFE column. Shading marks the physics-distilled block and bold the lowest error in each column of a block.}
\label{tab:1lyap}
\centering
\small
\setlength{\tabcolsep}{7pt}
\resizebox{\linewidth}{!}{\begin{tabular}{lrcccccccccc}
\toprule
 & & \multicolumn{2}{c}{Lorenz} & \multicolumn{2}{c}{Chen} & \multicolumn{2}{c}{R\"ossler} & \multicolumn{2}{c}{Thomas} & \multicolumn{2}{c}{Halvorsen} \\
\cmidrule(lr){3-4} \cmidrule(lr){5-6} \cmidrule(lr){7-8} \cmidrule(lr){9-10} \cmidrule(lr){11-12}
Method & NFE & in-dist & OOD & in-dist & OOD & in-dist & OOD & in-dist & OOD & in-dist & OOD \\
\midrule
\rowcolor{black!8}\multicolumn{12}{l}{\emph{Evolution operator}} \\
Data & 20 & 0.99{\scriptsize\color{black!55}$\pm$0.27} & 0.71{\scriptsize\color{black!55}$\pm$0.10} & 0.66{\scriptsize\color{black!55}$\pm$0.03} & 0.64{\scriptsize\color{black!55}$\pm$0.09} & 3.1{\scriptsize\color{black!55}$\pm$0.8} & 2.9{\scriptsize\color{black!55}$\pm$0.3} & 2.8{\scriptsize\color{black!55}$\pm$0.5} & 2.5{\scriptsize\color{black!55}$\pm$0.4} & 1.4{\scriptsize\color{black!55}$\pm$0.4} & 1.3{\scriptsize\color{black!55}$\pm$0.1} \\
Semigroup & 20 & 0.97{\scriptsize\color{black!55}$\pm$0.17} & 0.63{\scriptsize\color{black!55}$\pm$0.09} & 0.70{\scriptsize\color{black!55}$\pm$0.04} & 0.67{\scriptsize\color{black!55}$\pm$0.10} & 3.3{\scriptsize\color{black!55}$\pm$0.8} & 3.1{\scriptsize\color{black!55}$\pm$0.9} & 3.1{\scriptsize\color{black!55}$\pm$1.2} & 2.6{\scriptsize\color{black!55}$\pm$0.3} & 1.5{\scriptsize\color{black!55}$\pm$0.4} & 1.2{\scriptsize\color{black!55}$\pm$0.2} \\
Eulerian-self & 20 & 0.91{\scriptsize\color{black!55}$\pm$0.39} & 0.73{\scriptsize\color{black!55}$\pm$0.35} & 0.64{\scriptsize\color{black!55}$\pm$0.13} & 0.74{\scriptsize\color{black!55}$\pm$0.22} & 3.4{\scriptsize\color{black!55}$\pm$0.3} & 2.5{\scriptsize\color{black!55}$\pm$0.6} & 2.5{\scriptsize\color{black!55}$\pm$0.3} & 2.2{\scriptsize\color{black!55}$\pm$0.1} & 1.2{\scriptsize\color{black!55}$\pm$0.2} & 1.0{\scriptsize\color{black!55}$\pm$0.1} \\
Lagrangian-self & 20 & 0.71{\scriptsize\color{black!55}$\pm$0.19} & 0.83{\scriptsize\color{black!55}$\pm$0.19} & 1.1{\scriptsize\color{black!55}$\pm$0.4} & 0.86{\scriptsize\color{black!55}$\pm$0.24} & 3.5{\scriptsize\color{black!55}$\pm$1.1} & 2.8{\scriptsize\color{black!55}$\pm$0.5} & 2.4{\scriptsize\color{black!55}$\pm$0.7} & 2.4{\scriptsize\color{black!55}$\pm$0.3} & 1.5{\scriptsize\color{black!55}$\pm$0.7} & 1.5{\scriptsize\color{black!55}$\pm$0.2} \\
\rowcolor{physhl}Eulerian-phys & 20 & 0.36{\scriptsize\color{black!55}$\pm$0.21} & 0.44{\scriptsize\color{black!55}$\pm$0.30} & \textbf{0.18{\scriptsize\color{black!55}$\pm$0.01}} & \textbf{0.16{\scriptsize\color{black!55}$\pm$0.02}} & \textbf{1.4{\scriptsize\color{black!55}$\pm$0.1}} & \textbf{1.2{\scriptsize\color{black!55}$\pm$0.2}} & \textbf{0.79{\scriptsize\color{black!55}$\pm$0.04}} & \textbf{0.82{\scriptsize\color{black!55}$\pm$0.10}} & \textbf{0.36{\scriptsize\color{black!55}$\pm$0.02}} & \textbf{0.38{\scriptsize\color{black!55}$\pm$0.05}} \\
\rowcolor{physhl}Lagrangian-phys & 20 & \textbf{0.29{\scriptsize\color{black!55}$\pm$0.05}} & \textbf{0.31{\scriptsize\color{black!55}$\pm$0.13}} & 0.25{\scriptsize\color{black!55}$\pm$0.04} & 0.24{\scriptsize\color{black!55}$\pm$0.03} & 3.0{\scriptsize\color{black!55}$\pm$1.6} & 2.0{\scriptsize\color{black!55}$\pm$0.2} & 0.93{\scriptsize\color{black!55}$\pm$0.14} & 0.86{\scriptsize\color{black!55}$\pm$0.16} & 0.57{\scriptsize\color{black!55}$\pm$0.02} & 0.51{\scriptsize\color{black!55}$\pm$0.08} \\
\midrule
\rowcolor{black!8}\multicolumn{12}{l}{\emph{Baselines}} \\
AR $v$-pred & 200 & \textbf{0.22{\scriptsize\color{black!55}$\pm$0.07}} & \textbf{0.19{\scriptsize\color{black!55}$\pm$0.01}} & \textbf{0.32{\scriptsize\color{black!55}$\pm$0.09}} & \textbf{0.36{\scriptsize\color{black!55}$\pm$0.10}} & \textbf{0.53{\scriptsize\color{black!55}$\pm$0.14}} & \textbf{0.59{\scriptsize\color{black!55}$\pm$0.10}} & 2.6{\scriptsize\color{black!55}$\pm$1.2} & 2.2{\scriptsize\color{black!55}$\pm$0.8} & \textbf{0.29{\scriptsize\color{black!55}$\pm$0.01}} & \textbf{0.29{\scriptsize\color{black!55}$\pm$0.04}} \\
AR $x$-pred (cond.\ $t$) & 200 & 2.7{\scriptsize\color{black!55}$\pm$1.2} & 1.3{\scriptsize\color{black!55}$\pm$0.3} & 2.3{\scriptsize\color{black!55}$\pm$0.8} & 1.8{\scriptsize\color{black!55}$\pm$0.6} & 9.1{\scriptsize\color{black!55}$\pm$2.6} & 3.3{\scriptsize\color{black!55}$\pm$0.4} & 12{\scriptsize\color{black!55}$\pm$3} & 11{\scriptsize\color{black!55}$\pm$2} & 3.4{\scriptsize\color{black!55}$\pm$1.1} & 2.5{\scriptsize\color{black!55}$\pm$0.5} \\
AR $x$-pred (uncond.\ $t$) & 200 & 2.2{\scriptsize\color{black!55}$\pm$1.7} & 2.4{\scriptsize\color{black!55}$\pm$1.0} & 2.6{\scriptsize\color{black!55}$\pm$0.3} & 2.9{\scriptsize\color{black!55}$\pm$1.7} & 2.6{\scriptsize\color{black!55}$\pm$0.9} & 3.8{\scriptsize\color{black!55}$\pm$0.2} & 15{\scriptsize\color{black!55}$\pm$5} & 6.3{\scriptsize\color{black!55}$\pm$1.5} & 4.1{\scriptsize\color{black!55}$\pm$1.7} & 2.3{\scriptsize\color{black!55}$\pm$0.6} \\
CFO & 800 & 0.32{\scriptsize\color{black!55}$\pm$0.03} & 0.24{\scriptsize\color{black!55}$\pm$0.03} & 0.39{\scriptsize\color{black!55}$\pm$0.11} & \textbf{0.36{\scriptsize\color{black!55}$\pm$0.12}} & 2.5{\scriptsize\color{black!55}$\pm$1.0} & 1.6{\scriptsize\color{black!55}$\pm$0.4} & \textbf{1.6{\scriptsize\color{black!55}$\pm$0.3}} & \textbf{1.3{\scriptsize\color{black!55}$\pm$0.2}} & 0.33{\scriptsize\color{black!55}$\pm$0.06} & 0.41{\scriptsize\color{black!55}$\pm$0.01} \\
\bottomrule
\end{tabular}
}%
\end{table}

\FloatBarrier
\subsection{Per-system results}

\cref{tab:main-indist,tab:main-ood} list the numbers behind \cref{tab:main-avg}, one for the models trained on the whole horizon and one for the models trained on half of it.
Each gives all five metrics on each of the five systems.
Besides finding that physics-distillation is a better teacher for learning this evolution operator, we also find that AR $v$-pred is a good choice for autoregressive rollout on chaotic dynamics.

\vspace{\stretch{1}}
\begin{table}[!h]
\caption{\textbf{In-distribution results.}
Five metrics on five chaotic systems, averaged over three seeds, with $D_{\mathrm{H}}$ and $\edtwo$ in $10^{-2}$ and $\etrans$ in $10^{-3}$.
Shading marks the physics-distilled rows and bold the best method in each column of a block.
\cref{app:metrics,app:details} define the metrics and the training setup.}
\label{tab:main-indist}
\centering
\scriptsize
\setlength{\tabcolsep}{2.32pt}
\sbox0{%
\begin{tabular}{lrccccccccccccccc}
\toprule
 & & \multicolumn{5}{c}{Lorenz} & \multicolumn{5}{c}{Chen} & \multicolumn{5}{c}{R\"ossler} \\
\cmidrule(lr){3-7} \cmidrule(lr){8-12} \cmidrule(lr){13-17}
Method & NFE & $\rell$ & $D_{\mathrm{stsp}}$ & $D_{\mathrm{H}}$ & $\edtwo$ & $\etrans$ & $\rell$ & $D_{\mathrm{stsp}}$ & $D_{\mathrm{H}}$ & $\edtwo$ & $\etrans$ & $\rell$ & $D_{\mathrm{stsp}}$ & $D_{\mathrm{H}}$ & $\edtwo$ & $\etrans$ \\
\midrule
\rowcolor{black!8}\multicolumn{17}{l}{\emph{Evolution operator}} \\
Data & 100 & 0.32 & 0.36 & 1.2 & 0.52 & 0.62 & 0.26 & 0.63 & 2.0 & 1.3 & 0.55 & 0.53 & 0.19 & 8.0 & 29 & 28 \\
Semigroup & 100 & 0.30 & \textbf{0.32} & 1.2 & 0.83 & 0.65 & 0.27 & 0.60 & 2.1 & 1.4 & 0.70 & 0.54 & 0.30 & 10 & 40 & 27 \\
Eulerian-self & 100 & 0.27 & 0.39 & 1.2 & \textbf{0.31} & 0.46 & 0.27 & \textbf{0.58} & 2.2 & 0.77 & 0.63 & 0.56 & 0.31 & 10 & 45 & 14 \\
Lagrangian-self & 100 & 0.22 & 0.37 & 1.1 & 0.59 & 0.49 & 0.31 & 0.62 & 1.9 & 1.4 & \textbf{0.40} & 0.56 & 0.27 & 10 & 40 & 13 \\
\rowcolor{physhl}Eulerian-phys & 100 & \textbf{0.14} & 0.37 & \textbf{1.0} & 0.47 & 0.47 & \textbf{0.15} & 0.61 & 1.6 & \textbf{0.46} & 0.41 & \textbf{0.45} & \textbf{0.10} & \textbf{2.4} & \textbf{5.2} & \textbf{2.8} \\
\rowcolor{physhl}Lagrangian-phys & 100 & \textbf{0.14} & 0.38 & \textbf{1.0} & 0.48 & \textbf{0.43} & \textbf{0.15} & 0.60 & \textbf{1.4} & 0.78 & 0.60 & 0.52 & 0.38 & 9.7 & 51 & 8.0 \\
\midrule
\rowcolor{black!8}\multicolumn{17}{l}{\emph{Baselines}} \\
AR $v$-pred & 1000 & \textbf{0.11} & 0.38 & \textbf{1.0} & 0.46 & 0.61 & \textbf{0.19} & 0.62 & 1.8 & \textbf{0.59} & \textbf{0.54} & \textbf{0.27} & \textbf{0.096} & \textbf{1.6} & \textbf{1.4} & \textbf{0.22} \\
AR $x$-pred (cond.\ $t$) & 1000 & 0.50 & \textbf{0.36} & 1.3 & 0.49 & 0.87 & 0.45 & 0.69 & 3.6 & 1.2 & 0.78 & 0.61 & 0.17 & 15 & 35 & 8.7 \\
AR $x$-pred (uncond.\ $t$) & 1000 & 0.42 & 0.37 & 1.3 & \textbf{0.30} & \textbf{0.47} & 0.47 & 0.66 & 2.9 & 1.1 & 1.4 & 0.49 & 0.099 & 6.0 & 5.8 & 2.0 \\
CFO & 4000 & 0.14 & \textbf{0.36} & 1.1 & 0.77 & 0.65 & \textbf{0.19} & \textbf{0.58} & \textbf{1.7} & 0.71 & 0.73 & 0.47 & 0.11 & 3.5 & 2.6 & 0.74 \\
\bottomrule
\end{tabular}%
}%
\sbox2{%
\begin{tabular}{lrcccccccccc}
\toprule
 & & \multicolumn{5}{c}{Thomas} & \multicolumn{5}{c}{Halvorsen} \\
\cmidrule(lr){3-7} \cmidrule(lr){8-12}
Method & NFE & $\rell$ & $D_{\mathrm{stsp}}$ & $D_{\mathrm{H}}$ & $\edtwo$ & $\etrans$ & $\rell$ & $D_{\mathrm{stsp}}$ & $D_{\mathrm{H}}$ & $\edtwo$ & $\etrans$ \\
\midrule
\rowcolor{black!8}\multicolumn{12}{l}{\emph{Evolution operator}} \\
Data & 100 & 0.90 & 0.14 & 5.1 & 2.2 & \textbf{1.8} & 0.56 & 0.048 & 3.3 & 2.2 & 0.89 \\
Semigroup & 100 & 0.88 & \textbf{0.13} & 4.5 & 2.1 & 2.4 & 0.55 & 0.050 & 3.7 & 3.9 & 0.74 \\
Eulerian-self & 100 & 0.89 & 0.16 & 6.4 & 27 & 3.2 & 0.50 & \textbf{0.046} & 3.2 & 2.6 & 0.86 \\
Lagrangian-self & 100 & 0.85 & 0.15 & 5.6 & 1.7 & 2.6 & 0.55 & 0.049 & 3.8 & 6.0 & 1.1 \\
\rowcolor{physhl}Eulerian-phys & 100 & 0.64 & \textbf{0.13} & 3.8 & 0.85 & 2.2 & \textbf{0.31} & \textbf{0.046} & \textbf{2.1} & \textbf{0.75} & 0.57 \\
\rowcolor{physhl}Lagrangian-phys & 100 & \textbf{0.62} & \textbf{0.13} & \textbf{3.5} & \textbf{0.73} & 2.2 & 0.39 & 0.047 & 2.7 & 1.7 & \textbf{0.54} \\
\midrule
\rowcolor{black!8}\multicolumn{12}{l}{\emph{Baselines}} \\
AR $v$-pred & 1000 & 0.82 & 0.14 & 4.5 & 1.1 & \textbf{1.2} & 0.27 & \textbf{0.045} & 1.3 & 3.0 & \textbf{0.28} \\
AR $x$-pred (cond.\ $t$) & 1000 & 1.1 & 2.2 & 9.3 & 75 & 6.1 & 0.68 & 0.052 & 3.6 & 2.9 & 1.2 \\
AR $x$-pred (uncond.\ $t$) & 1000 & 1.1 & 0.59 & 12 & 32 & 4.6 & 0.72 & 0.053 & 5.4 & 8.9 & 1.8 \\
CFO & 4000 & \textbf{0.80} & \textbf{0.13} & \textbf{3.4} & \textbf{0.31} & 1.6 & \textbf{0.25} & 0.048 & \textbf{1.1} & \textbf{2.8} & 0.68 \\
\bottomrule
\end{tabular}%
}%
\usebox0\par\vspace{6pt}
{\setlength{\tabcolsep}{\dimexpr\tabcolsep+(\wd0-\wd2)/24\relax}%
\begin{tabular}{lrcccccccccc}
\toprule
 & & \multicolumn{5}{c}{Thomas} & \multicolumn{5}{c}{Halvorsen} \\
\cmidrule(lr){3-7} \cmidrule(lr){8-12}
Method & NFE & $\rell$ & $D_{\mathrm{stsp}}$ & $D_{\mathrm{H}}$ & $\edtwo$ & $\etrans$ & $\rell$ & $D_{\mathrm{stsp}}$ & $D_{\mathrm{H}}$ & $\edtwo$ & $\etrans$ \\
\midrule
\rowcolor{black!8}\multicolumn{12}{l}{\emph{Evolution operator}} \\
Data & 100 & 0.90 & 0.14 & 5.1 & 2.2 & \textbf{1.8} & 0.56 & 0.048 & 3.3 & 2.2 & 0.89 \\
Semigroup & 100 & 0.88 & \textbf{0.13} & 4.5 & 2.1 & 2.4 & 0.55 & 0.050 & 3.7 & 3.9 & 0.74 \\
Eulerian-self & 100 & 0.89 & 0.16 & 6.4 & 27 & 3.2 & 0.50 & \textbf{0.046} & 3.2 & 2.6 & 0.86 \\
Lagrangian-self & 100 & 0.85 & 0.15 & 5.6 & 1.7 & 2.6 & 0.55 & 0.049 & 3.8 & 6.0 & 1.1 \\
\rowcolor{physhl}Eulerian-phys & 100 & 0.64 & \textbf{0.13} & 3.8 & 0.85 & 2.2 & \textbf{0.31} & \textbf{0.046} & \textbf{2.1} & \textbf{0.75} & 0.57 \\
\rowcolor{physhl}Lagrangian-phys & 100 & \textbf{0.62} & \textbf{0.13} & \textbf{3.5} & \textbf{0.73} & 2.2 & 0.39 & 0.047 & 2.7 & 1.7 & \textbf{0.54} \\
\midrule
\rowcolor{black!8}\multicolumn{12}{l}{\emph{Baselines}} \\
AR $v$-pred & 1000 & 0.82 & 0.14 & 4.5 & 1.1 & \textbf{1.2} & 0.27 & \textbf{0.045} & 1.3 & 3.0 & \textbf{0.28} \\
AR $x$-pred (cond.\ $t$) & 1000 & 1.1 & 2.2 & 9.3 & 75 & 6.1 & 0.68 & 0.052 & 3.6 & 2.9 & 1.2 \\
AR $x$-pred (uncond.\ $t$) & 1000 & 1.1 & 0.59 & 12 & 32 & 4.6 & 0.72 & 0.053 & 5.4 & 8.9 & 1.8 \\
CFO & 4000 & \textbf{0.80} & \textbf{0.13} & \textbf{3.4} & \textbf{0.31} & 1.6 & \textbf{0.25} & 0.048 & \textbf{1.1} & \textbf{2.8} & 0.68 \\
\bottomrule
\end{tabular}%
}%
\end{table}
\vspace{\stretch{3}}

\begin{table}[p]
\vbox to \textheight{\vfill%
\caption{\textbf{Out-of-distribution results.}
The same five metrics on the five systems, averaged over three seeds, with $D_{\mathrm{H}}$ and $\edtwo$ in $10^{-2}$ and $\etrans$ in $10^{-3}$.
Every model is trained on the first half of the horizon and rolled out to all of it.
Shading marks the physics-distilled rows and bold the best method in each column of a block.
\cref{app:metrics,app:details} define the metrics and the training setup.}
\label{tab:main-ood}
\centering
\scriptsize
\setlength{\tabcolsep}{2.32pt}
\sbox0{%
\begin{tabular}{lrccccccccccccccc}
\toprule
 & & \multicolumn{5}{c}{Lorenz} & \multicolumn{5}{c}{Chen} & \multicolumn{5}{c}{R\"ossler} \\
\cmidrule(lr){3-7} \cmidrule(lr){8-12} \cmidrule(lr){13-17}
Method & NFE & $\rell$ & $D_{\mathrm{stsp}}$ & $D_{\mathrm{H}}$ & $\edtwo$ & $\etrans$ & $\rell$ & $D_{\mathrm{stsp}}$ & $D_{\mathrm{H}}$ & $\edtwo$ & $\etrans$ & $\rell$ & $D_{\mathrm{stsp}}$ & $D_{\mathrm{H}}$ & $\edtwo$ & $\etrans$ \\
\midrule
\rowcolor{black!8}\multicolumn{17}{l}{\emph{Evolution operator}} \\
Data & 100 & 0.25 & 0.38 & 1.0 & \textbf{0.28} & 0.45 & 0.27 & 0.61 & 2.8 & 0.67 & 0.98 & 0.53 & 0.14 & 7.7 & 27 & 4.9 \\
Semigroup & 100 & 0.24 & 0.39 & 1.2 & 0.31 & 0.52 & 0.29 & 0.63 & 2.2 & 0.78 & 0.71 & 0.54 & 0.12 & 6.7 & 14 & 3.7 \\
Eulerian-self & 100 & 0.24 & 0.36 & 1.2 & 0.72 & 0.55 & 0.30 & \textbf{0.57} & 2.2 & \textbf{0.38} & 1.4 & 0.50 & 0.11 & 3.3 & 10 & 6.1 \\
Lagrangian-self & 100 & 0.31 & \textbf{0.35} & 1.1 & 0.38 & 0.44 & 0.32 & 0.60 & 2.5 & 0.90 & \textbf{0.45} & 0.54 & 0.13 & 4.8 & 19 & \textbf{3.5} \\
\rowcolor{physhl}Eulerian-phys & 100 & 0.13 & \textbf{0.35} & 1.1 & 0.38 & \textbf{0.37} & \textbf{0.13} & 0.58 & \textbf{1.5} & 0.50 & 0.65 & \textbf{0.44} & 0.12 & \textbf{2.4} & \textbf{4.2} & 3.9 \\
\rowcolor{physhl}Lagrangian-phys & 100 & \textbf{0.12} & 0.42 & \textbf{0.95} & 0.53 & 0.69 & 0.17 & 0.63 & 1.7 & 0.78 & 0.63 & 0.51 & \textbf{0.10} & 2.9 & 8.2 & 4.8 \\
\midrule
\rowcolor{black!8}\multicolumn{17}{l}{\emph{Baselines}} \\
AR $v$-pred & 1000 & \textbf{0.094} & \textbf{0.39} & \textbf{1.0} & 1.1 & 0.57 & \textbf{0.21} & \textbf{0.57} & \textbf{1.7} & \textbf{0.91} & 0.64 & \textbf{0.30} & 0.098 & \textbf{1.7} & \textbf{0.78} & \textbf{0.26} \\
AR $x$-pred (cond.\ $t$) & 1000 & 0.75 & 64 & 41 & 32 & 21 & 1.0 & 1.5 & 34 & 21 & 19 & 0.71 & \textbf{0.089} & 8.4 & 8.7 & 1.7 \\
AR $x$-pred (uncond.\ $t$) & 1000 & 0.39 & 0.41 & 1.3 & \textbf{0.34} & \textbf{0.29} & 0.45 & 0.65 & 3.7 & 1.5 & 1.5 & 0.57 & 0.093 & 5.8 & 4.0 & 2.9 \\
CFO & 4000 & 0.55 & 0.44 & 3.8 & 1.0 & 2.2 & 0.30 & 0.67 & 2.8 & 1.1 & \textbf{0.39} & 8.7 & 24 & 81 & 71 & 150 \\
\bottomrule
\end{tabular}%
}%
\sbox2{%
\begin{tabular}{lrcccccccccc}
\toprule
 & & \multicolumn{5}{c}{Thomas} & \multicolumn{5}{c}{Halvorsen} \\
\cmidrule(lr){3-7} \cmidrule(lr){8-12}
Method & NFE & $\rell$ & $D_{\mathrm{stsp}}$ & $D_{\mathrm{H}}$ & $\edtwo$ & $\etrans$ & $\rell$ & $D_{\mathrm{stsp}}$ & $D_{\mathrm{H}}$ & $\edtwo$ & $\etrans$ \\
\midrule
\rowcolor{black!8}\multicolumn{12}{l}{\emph{Evolution operator}} \\
Data & 100 & 0.92 & 0.13 & 4.3 & 1.1 & \textbf{0.88} & 0.52 & 0.047 & \textbf{2.2} & 1.7 & 0.90 \\
Semigroup & 100 & 0.90 & 0.14 & 5.2 & 0.90 & 2.1 & 0.52 & 0.047 & 2.5 & \textbf{1.1} & 1.1 \\
Eulerian-self & 100 & 0.83 & 0.13 & 4.8 & 1.2 & 1.9 & 0.49 & 0.048 & 3.0 & 3.1 & 0.53 \\
Lagrangian-self & 100 & 0.87 & 0.14 & 4.1 & 0.77 & 1.5 & 0.56 & 0.049 & 4.0 & 5.8 & 0.72 \\
\rowcolor{physhl}Eulerian-phys & 100 & \textbf{0.61} & \textbf{0.12} & 3.7 & \textbf{0.46} & 2.3 & \textbf{0.32} & 0.047 & 2.5 & 1.3 & 0.84 \\
\rowcolor{physhl}Lagrangian-phys & 100 & 0.62 & 0.13 & \textbf{2.9} & 1.0 & 1.7 & 0.37 & \textbf{0.044} & 2.3 & 1.3 & \textbf{0.42} \\
\midrule
\rowcolor{black!8}\multicolumn{12}{l}{\emph{Baselines}} \\
AR $v$-pred & 1000 & \textbf{0.78} & \textbf{0.13} & \textbf{3.5} & \textbf{1.0} & \textbf{1.4} & \textbf{0.25} & \textbf{0.046} & \textbf{1.1} & \textbf{3.3} & \textbf{0.23} \\
AR $x$-pred (cond.\ $t$) & 1000 & 2.1 & 14 & 34 & 84 & 25 & 1.8 & 6.3 & 44 & 110 & 150 \\
AR $x$-pred (uncond.\ $t$) & 1000 & 1.1 & 0.57 & 7.8 & 13 & 2.1 & 0.57 & 0.049 & 3.9 & 6.7 & 1.4 \\
CFO & 4000 & 2.1 & 22 & 38 & 88 & 8.9 & 0.88 & 0.10 & 8.5 & 17 & 2.3 \\
\bottomrule
\end{tabular}%
}%
\usebox0\par\vspace{6pt}
{\setlength{\tabcolsep}{\dimexpr\tabcolsep+(\wd0-\wd2)/24\relax}%
\begin{tabular}{lrcccccccccc}
\toprule
 & & \multicolumn{5}{c}{Thomas} & \multicolumn{5}{c}{Halvorsen} \\
\cmidrule(lr){3-7} \cmidrule(lr){8-12}
Method & NFE & $\rell$ & $D_{\mathrm{stsp}}$ & $D_{\mathrm{H}}$ & $\edtwo$ & $\etrans$ & $\rell$ & $D_{\mathrm{stsp}}$ & $D_{\mathrm{H}}$ & $\edtwo$ & $\etrans$ \\
\midrule
\rowcolor{black!8}\multicolumn{12}{l}{\emph{Evolution operator}} \\
Data & 100 & 0.92 & 0.13 & 4.3 & 1.1 & \textbf{0.88} & 0.52 & 0.047 & \textbf{2.2} & 1.7 & 0.90 \\
Semigroup & 100 & 0.90 & 0.14 & 5.2 & 0.90 & 2.1 & 0.52 & 0.047 & 2.5 & \textbf{1.1} & 1.1 \\
Eulerian-self & 100 & 0.83 & 0.13 & 4.8 & 1.2 & 1.9 & 0.49 & 0.048 & 3.0 & 3.1 & 0.53 \\
Lagrangian-self & 100 & 0.87 & 0.14 & 4.1 & 0.77 & 1.5 & 0.56 & 0.049 & 4.0 & 5.8 & 0.72 \\
\rowcolor{physhl}Eulerian-phys & 100 & \textbf{0.61} & \textbf{0.12} & 3.7 & \textbf{0.46} & 2.3 & \textbf{0.32} & 0.047 & 2.5 & 1.3 & 0.84 \\
\rowcolor{physhl}Lagrangian-phys & 100 & 0.62 & 0.13 & \textbf{2.9} & 1.0 & 1.7 & 0.37 & \textbf{0.044} & 2.3 & 1.3 & \textbf{0.42} \\
\midrule
\rowcolor{black!8}\multicolumn{12}{l}{\emph{Baselines}} \\
AR $v$-pred & 1000 & \textbf{0.78} & \textbf{0.13} & \textbf{3.5} & \textbf{1.0} & \textbf{1.4} & \textbf{0.25} & \textbf{0.046} & \textbf{1.1} & \textbf{3.3} & \textbf{0.23} \\
AR $x$-pred (cond.\ $t$) & 1000 & 2.1 & 14 & 34 & 84 & 25 & 1.8 & 6.3 & 44 & 110 & 150 \\
AR $x$-pred (uncond.\ $t$) & 1000 & 1.1 & 0.57 & 7.8 & 13 & 2.1 & 0.57 & 0.049 & 3.9 & 6.7 & 1.4 \\
CFO & 4000 & 2.1 & 22 & 38 & 88 & 8.9 & 0.88 & 0.10 & 8.5 & 17 & 2.3 \\
\bottomrule
\end{tabular}%
}%
\vfill}%
\end{table}

\FloatBarrier
\subsection{Error curves}

That endpoint is one point on a curve, and \cref{fig:err-curves} shows the whole of it, in distribution and out of it.
Every model runs at its own step size, so the evolution operator spends ten times fewer evaluations than AR $v$-pred and forty times fewer than CFO.

\begin{figure}[h]
\centering
\includegraphics[width=\linewidth, trim=0 26.5 0 7.0, clip]{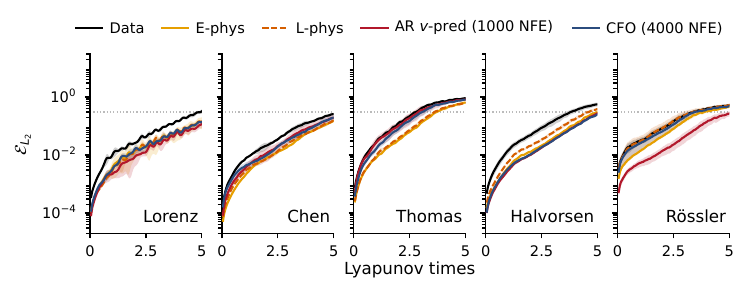}\\[1pt]
\includegraphics[width=\linewidth, trim=0 7.4 0 21.0, clip]{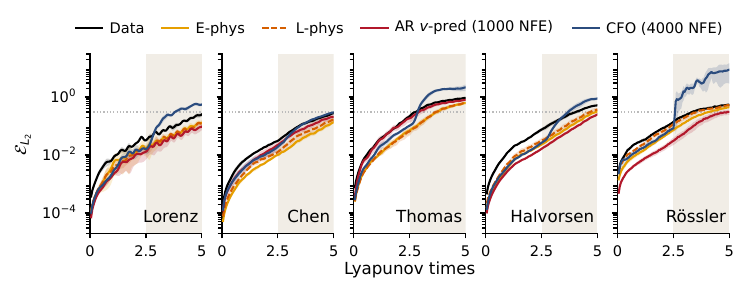}
\caption{\textbf{Error curves.}
$\rell$ against time on the five systems, averaged over three seeds with one-standard-deviation bands, and the dotted line is the relative error $0.3$.
Top: in distribution.
Bottom: out of distribution, with the half beyond the training window shaded.}
\label{fig:err-curves}
\end{figure}

\FloatBarrier
\subsection{Combined differential identities}
\label{app:el}

We also try to use both differential identities of \cref{eq:eul,eq:lag} for training the evolution operator, as follows:
\begin{equation}
\mathcal{L}(\theta)=\mathcal{L}_{\mathrm{data}}(\theta)+w_1\,\mathcal{L}^{(p)}_{\mathrm{E}}(\theta)+w_2\,\mathcal{L}^{(p)}_{\mathrm{L}}(\theta),
\qquad w_1,w_2\ge 0,
\label{eq:loss-el}
\end{equation}
with the two residuals of \cref{eq:loss-eulerian,eq:loss-lagrangian} at $p=2$.
Vizier searches both weights on the validation set over more trials than for one identity, with every other setting of \cref{app:details} unchanged.
\cref{tab:el} compares all the original Eulerian and Lagrangian objectives with their combined version.

\begin{table}[!h]
\caption{\textbf{Combined differential identities.}
$\rell(5\Tlyap)$ on each of the five systems, mean $\pm$ std over three seeds, for each differential identity on its own and for the combination of \cref{eq:loss-el}. Shading marks the physics-distilled block and bold the lowest error in each column of a block.}
\label{tab:el}
\centering
\small
\setlength{\tabcolsep}{7pt}
\resizebox{\linewidth}{!}{\begin{tabular}{lcccccccccc}
\toprule
 & \multicolumn{2}{c}{Lorenz} & \multicolumn{2}{c}{Chen} & \multicolumn{2}{c}{R\"ossler} & \multicolumn{2}{c}{Thomas} & \multicolumn{2}{c}{Halvorsen} \\
\cmidrule(lr){2-3} \cmidrule(lr){4-5} \cmidrule(lr){6-7} \cmidrule(lr){8-9} \cmidrule(lr){10-11}
Consistency term & in-dist & OOD & in-dist & OOD & in-dist & OOD & in-dist & OOD & in-dist & OOD \\
\midrule
\rowcolor{black!8}\multicolumn{11}{l}{\emph{Self-distillation}} \\
Eulerian & 0.27{\scriptsize\color{black!55}$\pm$0.06} & \textbf{0.24{\scriptsize\color{black!55}$\pm$0.01}} & \textbf{0.27{\scriptsize\color{black!55}$\pm$0.03}} & 0.30{\scriptsize\color{black!55}$\pm$0.04} & 0.56{\scriptsize\color{black!55}$\pm$0.02} & \textbf{0.50{\scriptsize\color{black!55}$\pm$0.04}} & 0.89{\scriptsize\color{black!55}$\pm$0.07} & \textbf{0.83{\scriptsize\color{black!55}$\pm$0.05}} & 0.50{\scriptsize\color{black!55}$\pm$0.02} & 0.49{\scriptsize\color{black!55}$\pm$0.02} \\
Lagrangian & 0.22{\scriptsize\color{black!55}$\pm$0.03} & 0.31{\scriptsize\color{black!55}$\pm$0.14} & 0.31{\scriptsize\color{black!55}$\pm$0.05} & 0.32{\scriptsize\color{black!55}$\pm$0.06} & 0.56{\scriptsize\color{black!55}$\pm$0.06} & 0.54{\scriptsize\color{black!55}$\pm$0.03} & 0.85{\scriptsize\color{black!55}$\pm$0.15} & 0.87{\scriptsize\color{black!55}$\pm$0.04} & 0.55{\scriptsize\color{black!55}$\pm$0.07} & 0.56{\scriptsize\color{black!55}$\pm$0.02} \\
Eulerian $+$ Lagrangian & \textbf{0.21{\scriptsize\color{black!55}$\pm$0.04}} & 0.31{\scriptsize\color{black!55}$\pm$0.13} & \textbf{0.27{\scriptsize\color{black!55}$\pm$0.04}} & \textbf{0.27{\scriptsize\color{black!55}$\pm$0.03}} & \textbf{0.55{\scriptsize\color{black!55}$\pm$0.03}} & \textbf{0.50{\scriptsize\color{black!55}$\pm$0.02}} & \textbf{0.81{\scriptsize\color{black!55}$\pm$0.05}} & 0.84{\scriptsize\color{black!55}$\pm$0.07} & \textbf{0.49{\scriptsize\color{black!55}$\pm$0.03}} & \textbf{0.46{\scriptsize\color{black!55}$\pm$0.02}} \\
\midrule
\rowcolor{black!8}\multicolumn{11}{l}{\emph{Physics-distillation}} \\
\rowcolor{physhl}Eulerian & 0.14{\scriptsize\color{black!55}$\pm$0.06} & 0.13{\scriptsize\color{black!55}$\pm$0.03} & \textbf{0.15{\scriptsize\color{black!55}$\pm$0.01}} & \textbf{0.13{\scriptsize\color{black!55}$\pm$0.02}} & \textbf{0.45{\scriptsize\color{black!55}$\pm$0.03}} & \textbf{0.44{\scriptsize\color{black!55}$\pm$0.03}} & 0.64{\scriptsize\color{black!55}$\pm$0.07} & \textbf{0.61{\scriptsize\color{black!55}$\pm$0.06}} & 0.31{\scriptsize\color{black!55}$\pm$0.01} & 0.32{\scriptsize\color{black!55}$\pm$0.02} \\
\rowcolor{physhl}Lagrangian & 0.14{\scriptsize\color{black!55}$\pm$0.01} & 0.12{\scriptsize\color{black!55}$\pm$0.02} & \textbf{0.15{\scriptsize\color{black!55}$\pm$0.01}} & 0.17{\scriptsize\color{black!55}$\pm$0.04} & 0.52{\scriptsize\color{black!55}$\pm$0.06} & 0.51{\scriptsize\color{black!55}$\pm$0.04} & 0.62{\scriptsize\color{black!55}$\pm$0.02} & 0.62{\scriptsize\color{black!55}$\pm$0.05} & 0.39{\scriptsize\color{black!55}$\pm$0.03} & 0.37{\scriptsize\color{black!55}$\pm$0.05} \\
\rowcolor{physhl}Eulerian $+$ Lagrangian & \textbf{0.11{\scriptsize\color{black!55}$\pm$0.02}} & \textbf{0.11{\scriptsize\color{black!55}$\pm$0.03}} & 0.16{\scriptsize\color{black!55}$\pm$0.02} & 0.16{\scriptsize\color{black!55}$\pm$0.05} & 0.47{\scriptsize\color{black!55}$\pm$0.06} & 0.50{\scriptsize\color{black!55}$\pm$0.02} & \textbf{0.61{\scriptsize\color{black!55}$\pm$0.05}} & \textbf{0.61{\scriptsize\color{black!55}$\pm$0.02}} & \textbf{0.30{\scriptsize\color{black!55}$\pm$0.04}} & \textbf{0.29{\scriptsize\color{black!55}$\pm$0.02}} \\
\bottomrule
\end{tabular}
}
\end{table}

\FloatBarrier

\FloatBarrier
\clearpage
\section{Related work}
\label{app:related}
\paragraph{Learning dynamics from trajectories.}
Autoregressive models~\citep{li2020fourier, brandstetter2022message, ye2025recurrentneuraloperatorsstable} and flow-map learning methods~\citep{qin2019data} learn a fixed-step update on a uniform time grid and roll it out.
Neural ODEs~\citep{chen2018neural} and CFO~\citep{hou2025cfolearningcontinuoustimepde} learn the instantaneous velocity instead.
CFO fits a spline through each trajectory and matches its time derivative by flow matching, so it trains on irregular time grids without backpropagating through a solver.
Both families still advance in small steps at inference.

\paragraph{Accelerating many-step models.}
Consistency models~\citep{song2023consistencymodels}, shortcut models~\citep{frans2025shortcut}, MeanFlow~\citep{geng2025meanflowsonestepgenerative}, Terminal Velocity Matching~\citep{zhou2026terminal}, and flow maps~\citep{boffi2025flowmapmatchingstochastic, boffi2025consistency} replace many sampling steps with one or a few by learning the map between two generative times.
Spatio-Temporal MeanFlow~\citep{bai2026physics} carries this map into physical time with the governing equations, between absolute start and end times.
Hamiltonian Flow Maps~\citep{ripken2026hamiltonian} instead learn it over the elapsed time from the governing equations alone, in the Eulerian form.
\mbox{Deep-OSG}~\citep{chen2023deeposg}, Consistent Vector Flow~\citep{luo2026recovering}, and a semigroup consistency diagnostic~\citep{shikhman2026semigroup} build on the semigroup property.
We derive all three consistency conditions of flow maps for one time-invariant evolution operator of an autonomous system in physical time, which unifies these elapsed-time works.
We take the velocity from either the model or the governing equations, providing a more general paradigm for distillation and physics-informed learning.

\section{Connection to other methods}
\label{app:connections}
\begingroup
\setlength{\abovedisplayskip}{5pt}\setlength{\belowdisplayskip}{5pt}
\setlength{\abovedisplayshortskip}{3pt}\setlength{\belowdisplayshortskip}{3pt}
Here we connect to previous work on accelerating many-step sampling models into fewer evaluations, including flow maps for generative modeling and Hamiltonian flow maps for molecular dynamics.

\paragraph{Flow maps.}
Flow maps \citep{boffi2025consistency, boffi2025flowmapmatchingstochastic} learn the map $X_{s,t}$ that carries a sample at generative time $s$ to time $t$ along a drift $b_t$, written as $X_{s,t}(x) = x + (t-s)\,v_{s,t}(x)$ through an average velocity, and three conditions characterize it:
\begin{equation}
\left\{
\begin{gathered}
\partial_t X_{s,t}(x) = b_t\big(X_{s,t}(x)\big), \\
\partial_s X_{s,t}(x) + \nabla X_{s,t}(x)\,b_s(x) = 0, \\
X_{u,t}\big(X_{s,u}(x)\big) = X_{s,t}(x).
\end{gathered}
\right.
\label{eq:flowmap-three}
\end{equation}
An autonomous system runs on physical time and is time-invariant: its drift is a fixed field $v$, so $X_{s,t}$ depends on $t-s$ rather than the absolute times $s$ and $t$.
Thus principally, the substitution
\begin{equation}
\left\{
\begin{gathered}
\psi(u,\Delta t) = X_{s,\,s+\Delta t}(u), \\
\bar{v}(u,\Delta t) = v_{s,\,s+\Delta t}(u), \\
b_t \equiv v,
\end{gathered}
\right.
\label{eq:flowmap-sub}
\end{equation}
carries \cref{eq:flowmap-three} onto it: with $\partial_t X_{s,t} = \partial_{\Delta t}\psi$ and $\partial_s X_{s,t} = -\partial_{\Delta t}\psi$, the first two become \cref{eq:lag} and \cref{eq:eul}, and the third becomes \cref{eq:semi}.
Our contribution is to extend flow-map consistency learning to autonomous physical dynamics and explore the design space, with the governing equations providing supervision through the differential consistency conditions.

\paragraph{Hamiltonian flow maps.}
Hamiltonian Flow Maps (HFM) \citep{ripken2026hamiltonian} learn the same object on the phase space of a system with Hamiltonian $\mathcal{H}$, writing the map as $(x,p)^{\mathsf{T}} + \Delta t\,\bar{u}(x,p,\Delta t)$, where the mean displacement field $\bar{u}$ averages $(v,f)^{\mathsf{T}}$ over the elapsed time, and impose the MeanFlow Identity \citep{geng2025meanflowsonestepgenerative}
\begin{equation}
\bar{u}(x,p,\Delta t) - \Delta t\,\tfrac{d}{dt}\bar{u}(x,p,\Delta t) = (v, f)^{\mathsf{T}},
\label{eq:hfm}
\end{equation}
where $v = \dot{x} = \partial\mathcal{H}/\partial p$ is the velocity and $f = \dot{p} = -\partial\mathcal{H}/\partial x$ the force.
Expanding the total derivative along the trajectory as $\tfrac{d}{dt}\bar{u} = D\bar{u}[(v,f)^{\mathsf{T}}] - \partial_{\Delta t}\bar{u}$ turns \cref{eq:hfm} into
\begin{equation}
\bar{u}(x,p,\Delta t) = (v,f)^{\mathsf{T}} + \Delta t\big(D\bar{u}(x,p,\Delta t)[(v,f)^{\mathsf{T}}] - \partial_{\Delta t}\bar{u}(x,p,\Delta t)\big),
\end{equation}
which is \cref{eq:eul} on the state $z = (x,p)^{\mathsf{T}}$ with $\mathcal{F}(z) = (v,f)^{\mathsf{T}}$, so HFM is our Eulerian-phys.
The main difference between the two is the supervision target.
HFM focuses (i) on $(x,p)$, a position and a momentum, and on their time derivatives $(v,f)$, a velocity and a force; (ii) without paired data for a data term.
In our setting, we extend it to general physics $\mathcal{F}$ and learn with a small amount of real paired data from trajectories.
The two supervisions are complementary, since $\mathcal{F}$ fixes the velocity at any state while the trajectories fix the endpoint of each jump.
\endgroup

\FloatBarrier
\clearpage
\section{Qualitative results}
\label{app:qualitative}
The tables report the long-run behavior as single numbers.
\cref{fig:longterm-qual} shows what those numbers measure: the attractor structure each model produces over a $50\Tlyap$ trajectory, plotted as a cloud, and the power spectrum of one of its coordinates.

\begin{figure}[h]
\centering
\includegraphics[width=\linewidth]{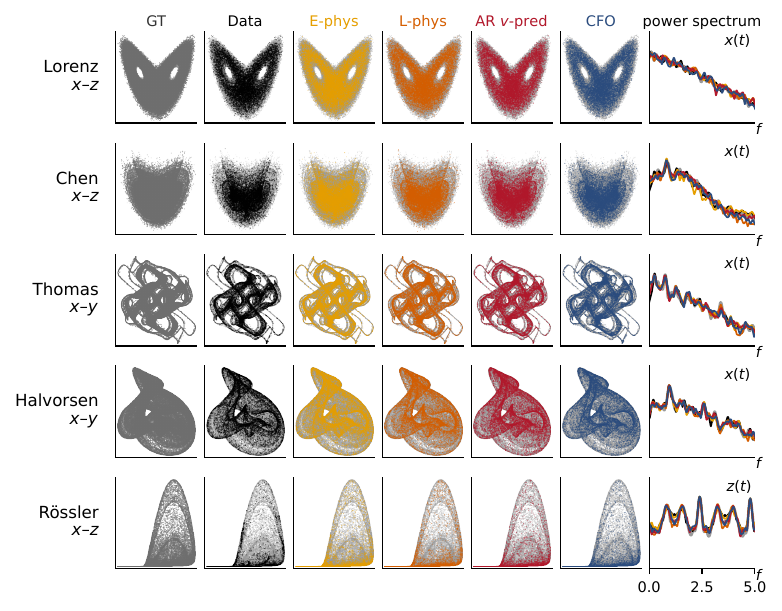}
\caption{\textbf{Long-term statistics on all five systems.}
One column per method on seed $0$, showing the attractor cloud of the $50\Tlyap$ free run with the truth in gray behind it, and the power spectrum of one coordinate on one trajectory.
The frequency $f$ is in units of the Lyapunov rate $\lambda_{\max}$.}
\label{fig:longterm-qual}
\end{figure}

\end{document}